\documentclass{tlp}

\usepackage{amsmath}
\usepackage{amssymb}
\usepackage{graphicx}
\usepackage{multirow}
\usepackage{thmtools,thm-restate}
\usepackage{tikz}
\usetikzlibrary{arrows, shapes.geometric,fit, positioning,calc}
\usepackage{natbib}
\usepackage{hyperref}

\usepackage{subcaption}
\usepackage{enumitem} 

\def\setaf{\textit{SETAF}}
\def\nlp{\textit{NLP}}
\def\rfalp{\textit{RFALP}}
\def\utpm{\textit{UTPM}}

\def\nadf{\mathit{ADF}^+}

\newcommand{\comments}[1]{}
\newcommand{\set}[1]{\left\{#1\right\}}
\newcommand{\pair}[1]{\left< #1 \right>}

\newcommand{\aaf}{\mathit{AAF}}
\newcommand{\ar}{\mathcal{A}}
\newcommand{\af}{\mathfrak{A}}
\newcommand{\att}{\mathit{Att}}

\newcommand{\inn}{\mathtt{in}}
\newcommand{\out}{\mathtt{out}}
\newcommand{\undec}{\mathtt{undec}}
\newcommand{\nlab}{\mathcal{L}}
\newcommand{\lab}{\mathcal{L}\mathit{ab}}
\newcommand{\HB}{\mathit{HB}}
\newcommand{\naf}{\mathtt{not\ }}
\newcommand{\head}{\mathit{Head}}
\newcommand{\hd}{\mathit{head}}
\newcommand{\body}{\mathit{body}}
\newcommand{\Conc}{\mathtt{Conc}}
\newcommand{\Vul}{\mathtt{Vul}}
\newcommand{\Rules}{\mathtt{Rules}}
\newcommand{\Sub}{\mathtt{Sub}}

\newcommand{\inter}{\mathcal{I}}
\newcommand{\model}{\mathcal{I}\mathit{nt}}

\newcommand{\LabMod}{\mathcal{L}2\mathcal{I}_P}
\newcommand{\ModLab}{\mathcal{I}2\mathcal{L}_P}
\newcommand{\lm}{\mathcal{L}2\mathcal{I}_\af}
\newcommand{\ml}{\mathcal{I}2\mathcal{L}_\af}

 \newtheorem{theo}{Theorem}
 \newtheorem{lem}[theo]{Lemma}
 
 \newtheorem{cor}[theo]{Corollary}
 \newtheorem{mdef}{Definition} 
 \newtheorem{rmk}{Remark}
 \newtheorem{example}{Example}

 \makeatletter
\tikzset{
    @pos/.style={@pos1={#1},@pos2={#1}},
    @ratio/.style={@ratio1={#1},@ratio2={#1}},
    @delta/.style={@delta1={#1},@delta2={#1}},
    @edge/.style={@@edge/.append style={#1}},
    @edge 0/.style={@@edge 0/.append style={#1}},
    @edge 1/.style={@@edge 1/.append style={#1}},
    @edge 2/.style={@@edge 2/.append style={#1}},
    @edge 3/.style={@@edge 3/.append style={#1}},
    @edge 4/.style={@@edge 4/.append style={#1}},
    @pos1/.store in=\qrr@posA,
    @pos2/.store in=\qrr@posB,
    @ratio1/.store in=\qrr@ratioA,
    @ratio2/.store in=\qrr@ratioB,
    @delta1/.store in=\qrr@deltaA,
    @delta2/.store in=\qrr@deltaB,
    @pos=.5,
    @ratio=.5,
    @delta=.1,
}

\newcommand*{\connectThreeLoop}[4][]{
    \begingroup
    \tikzset{#1}
    \coordinate (@aux1) at ($(#2)!\qrr@ratioA!(#3)$);
    \coordinate (@aux2) at ($(#4)!\qrr@posA!(@aux1)$);
    \path (@aux2) edge[@@edge/.try, @@edge 0/.try, @@edge 3/.try] (#4);
    \draw[@@edge/.try, @@edge 1/.try] (@aux2) .. controls ($(#4)!\qrr@posA+\qrr@deltaA!(@aux1)$) .. (#2);
    \draw[@@edge/.try, @@edge 2/.try] (@aux2) .. controls ($(#4)!\qrr@posA+\qrr@deltaA!(@aux1)$) .. (#3);
    \endgroup
}
\newcommand*{\connectThree}[4][]{
    \begingroup
    \tikzset{#1}
    \coordinate (@aux1) at ($(#2)!\qrr@ratioA!(#3)$);
    \coordinate (@aux2) at ($(#4)!\qrr@posA!(@aux1)$);
    \path (@aux2) edge[@@edge/.try, @@edge 0/.try, @@edge 3/.try] (#4);
    \draw[@@edge/.try, @@edge 1/.try] (@aux2) .. controls ($(#4)!\qrr@posA+\qrr@deltaA!(@aux1)$) .. (#2);
    \draw[@@edge/.try, @@edge 2/.try] (@aux2) .. controls ($(#4)!\qrr@posA+\qrr@deltaA!(@aux1)$) .. (#3);
    \endgroup
}
\newcommand*{\connectFour}[5][]{
    \begingroup
    \tikzset{#1}
    \coordinate (@aux1a) at ($(#2)!\qrr@ratioA!(#3)$);
    \coordinate (@aux1b) at ($(#4)!\qrr@ratioB!(#5)$);
    \coordinate (@aux2a) at ($(@aux1b)!\qrr@posA!(@aux1a)$);
    \coordinate (@aux2b) at ($(@aux1a)!\qrr@posB!(@aux1b)$);
    \path (@aux2a) edge[@@edge/.try,@@edge 0/.try] (@aux2b);
    \draw[@@edge/.try,@@edge 1/.try] (@aux2a) .. controls ($(@aux1b)!\qrr@posA+\qrr@deltaA!(@aux1a)$) .. (#2);
    \draw[@@edge/.try,@@edge 2/.try] (@aux2a) .. controls ($(@aux1b)!\qrr@posA+\qrr@deltaA!(@aux1a)$) .. (#3);
    \draw[@@edge/.try,@@edge 3/.try] (@aux2b) .. controls ($(@aux1a)!\qrr@posB+\qrr@deltaB!(@aux1b)$) .. (#4);
    \draw[@@edge/.try,@@edge 4/.try] (@aux2b) .. controls ($(@aux1a)!\qrr@posB+\qrr@deltaB!(@aux1b)$) .. (#5);
    \draw[help lines] (@aux1a) -- (@aux1b) node[midway,above,sloped,font=\tiny,shape=rectangle,inner xsep=+0pt,draw=none,align=center,fill=white,fill opacity=.75,outer ysep=\pgflinewidth,text opacity=1] {ratio: \qrr@ratioA/\qrr@ratioB\\pos: \qrr@posA/\qrr@posB\\delta: \qrr@deltaA/\qrr@deltaB};
    \endgroup
}
\makeatother

\begin{document}

\lefttitle{J. Alcântara,  R. Cordeiro and S. Sá}

\jnlPage{1}{8}
\jnlDoiYr{2021}
\doival{10.1017/xxxxx}

\title[On the Equivalence between Logic Programming and $\setaf$]{On the Equivalence between Logic Programming and $\setaf$}

\begin{authgrp}
\author{\sn{João} \gn{Alcântara}}
\affiliation{Federal University of Ceará, Brazil\\\email{jnando@dc.ufc.br}}
\author{\sn{Renan} \gn{Cordeiro}}
\affiliation{Federal University of Ceará, Brazil\\\email{renandcsc@alu.ufc.br}}
\author{\sn{Samy} \gn{Sá}}
\affiliation{Federal University of Ceará, Brazil\\\email{samy@ufc.br}}
\end{authgrp}


\maketitle

\begin{abstract}
A framework with sets of attacking arguments ($\setaf$) is an extension of the well-known Dung's Abstract Argumentation Frameworks ($\aaf$s) that allows joint attacks on arguments. In this paper, we provide a translation from Normal Logic Programs ($\nlp$s) to $\setaf$s and vice versa, from $\setaf$s to $\nlp$s. We show that there is pairwise equivalence between their semantics, including the equivalence between $L$-stable and semi-stable semantics. Furthermore, for a class of $\nlp$s called Redundancy-Free Atomic Logic Programs ($\rfalp$s), there is also a structural equivalence as these back-and-forth translations are each other's inverse. Then, we show that $\rfalp$s are as expressive as $\nlp$s by transforming any $\nlp$ into an equivalent $\rfalp$ through a series of program transformations already known in the literature. We also show that these program transformations are confluent, meaning that every $\nlp$ will be transformed into a unique $\rfalp$. The results presented in this paper enhance our understanding that $\nlp$s and $\setaf$s are essentially the same formalism. Under consideration in Theory and Practice of Logic Programming (TPLP).
\end{abstract}

\begin{keywords}
Abstract Argumentation, $\setaf$, Logic Programming Semantics, Program Transformations
\end{keywords}

\section{Introduction}\label{s:introduction}

Argumentation and logic programming are two of the most successful paradigms in artificial intelligence and knowledge representation. Argumentation revolves around the idea of constructing and evaluating arguments to determine the acceptability of a claim. It models complex reasoning by considering various pieces of evidence and their interrelationships, making it a powerful tool for handling uncertainty and conflicting information. On the other hand, logic programming provides a formalism for expressing knowledge and defining computational processes through a set of logical rules. 

In this scenario, the Abstract Argumentation Frameworks ($\aaf$s) proposed by Dung in his seminal paper \citep{dung1995acceptability} have exerted a dominant influence over the development of formal argumentation. We can depict such frameworks simply as a directed graph whose nodes represent arguments and edges represent the attack relation between them. Indeed, in $\aaf$s, the content of these arguments is not considered, and the attack relation stands as the unique relation. The simplicity and elegance of $\aaf$s have made them an appealing formalism for computational applications.

In Dung's proposal, the semantics for $\aaf$s are given in terms of extensions, which are sets of arguments satisfying certain criteria of acceptability. Naturally, different criteria of acceptability will lead to different extension-based semantics, including Dung's original concepts of complete, stable, preferred and grounded semantics \citep{dung1995acceptability}, and semi-stable semantics \citep{caminada2006semi,verheij1996two}. A richer characterisation based on labellings was proposed by Caminada and Gabbay \citep{caminada2009logical} to describe these semantics. Differently from extensions, which explicitly regard solely the accepted arguments, the labelling-based approach permits a more fine-grained setting, where each argument is assigned a label $\inn$, $\out$, or $\undec$. Intuitively, we accept an argument labelled as $\inn$, reject one labelled as $\out$, and consider one labelled as $\undec$ as undecided, meaning it is neither accepted nor rejected.

Despite providing distinct perspectives on reasoning and decision-making, argumentation and logic programming have clear connections. Indeed, we can see in Dung's work \citep{dung1995acceptability} how to translate a Normal Logic Program ($\nlp$) into an $\aaf$. Besides, the author proved that stable models (resp. the well-founded model) of an $\nlp$ correspond to stable extensions (resp. the grounded extension) of the associated $\aaf$. These results led to several studies concerning connections between argumentation and logic programming \citep{dung95argumentation,nieves2008preferred,wu2009complete,toni2011argumentation,dvorak2013making,caminada15equivalence,caminada2022comparing}. In particular, \citep{wu2009complete} established the equivalence between complete semantics and partial stable semantics. These semantics generalise a series of other relevant semantics for each system, as extensively documented in \citep{caminada15equivalence}. However, one equivalence formerly expected to hold remained elusive: the correspondence between the semi-stable semantics \citep{caminada2006semi} in $\aaf$  and the $L$-stable semantics in $\nlp$ \citep{eiter97partial} could not be attained. They even showed in \citep{caminada15equivalence} that with their proposed translation from $\nlp$s to $\aaf$s, there cannot be a semantics for $\aaf$s equivalent to $L$-stable semantics.

In \citep{caminada2017equivalence}, the authors showed how to translate Assumption-Based Argumentation (\textit{ABA}) \citep{bondarenko1997abstract,dung2009assumption,toni2014tutorial} to $\nlp$s and how this translation can be reapplied for a reverse translation from $\nlp$s to \textit{ABA}. 
Curiously, the problematic direction here is from \textit{ABA} to $\nlp$. 
In \citep{caminada2017equivalence}, they have shown that with their translation, there cannot be a semantics for $\nlp$s equivalent to the semi-stable semantics \citep{caminada2015difference,schulz2015logic} for \textit{ABA}.

Since then, a great effort has been made to identify paradigms where semi-stable and $L$-stable semantics are equivalent. In \citep{alcantara2019equivalence}, the strategy was to look for more expressive argumentation frameworks than $\aaf$s: Attacking Dialectical Frameworks, a support-free fragment of Abstract Dialectical Frameworks \citep{brewka2010abstract,brewka2013abstract}, a generalisation of $\aaf$s designed to express arbitrary relationships among arguments. A translation from $\nlp$ to $\nadf$ was proved in \citep{alcantara2019equivalence} to account for various equivalences between their semantics, including the definition of a semantics for $\nadf$ corresponding to the $L$-stable semantics for $\nlp$s.

In a similar vein, other relevant proposals explored the equivalence between $L$-stable and semi-stable semantics for Claim-augmented Argumentation Frameworks ($\mathit{CAF}$s) \citep{dvovrak2023claim,rapberger2020defining,rocha2022credal}, which are a generalisation of $\aaf$s where each argument is explicitly associated with a claim, and for Bipolar Argumentation Frameworks ($\mathit{BAF}$s) with conclusions \citep{rocha2022bipolar}, a generalisation of $\mathit{CAF}$s with the inclusion of an explicit notion of support between arguments. In both frameworks, the equivalence with $\nlp$s does not just involve their semantics; it is also structural as there is a one-to-one mapping from them to $\nlp$s. 

In \citep{sa2021abstract}, instead of looking for more expressive argumentation frameworks, the idea was to introduce more fine-grained semantics to deal with $\aaf$s. Then a five-valued setting was employed rather than the usual three-valued one. As in the previous cases, this approach also captures the correspondence between the semantics for $\aaf$s and $\nlp$s. Specifically, it captures the correspondence involving $L$-stable semantics.

The connections between $\mathit{ABA}$ and logic programming were later revisited in \citep{sa2019interpretations,sa2021assumption}, where they proposed a new translation from $\mathit{ABA}$ frameworks to $\nlp$s. The correspondence between their semantics (including $L$-stable) is obtained by selecting specific atoms in the characterisation of the $\nlp$ semantics.

In summary, in the connections between $\nlp$ and argumentation semantics, the Achilles' heel is the relation between $L$-stable and semi-stable semantics.

In this paper, we focus on the relationship between logic programming and $\setaf$ \citep{nielsen2006generalization}, an extension of Dung’s $\aaf$s to allow joint attacks on arguments. Following the strategy adopted in \citep{caminada15equivalence,alcantara2019equivalence}, we resort to the characterisation of the $\setaf$ semantics in terms of labellings \citep{flouris2019comprehensive}. As a starting point, we provide a mapping from $\nlp$s to $\setaf$s (and vice versa) and show that $\nlp$s and $\setaf$s are pairwise equivalent under various semantics, including the equivalence between $L$-stable and semi-stable. These results were inspired directly by two of our previous works: the equivalence between $\nlp$s and $\nadf$s \citep{alcantara2019equivalence}, and the equivalence between $\nadf$s and $\setaf$s \citep{alcantara2021equivalence}.

Furthermore, we investigate a class of $\nlp$s called Redundancy-Free Atomic Logic Programs ($\rfalp$s) \citep{konig22just}. In $\rfalp$s, the translations from $\nlp$s to $\setaf$s and vice versa preserve the structure of each other's theories. In essence, these translations become inverses of each other. Consequently, the equivalence results concerning $\nlp$s and $\setaf$s have deeper implications than the correspondence results between $\nlp$s and $\aaf$s: they encompass equivalence in both semantics as well as structure.

Some of these results are not new as recently they have already been obtained independently by König et al. \citep{konig22just}. In fact, their translation from $\nlp$s to $\setaf$s and vice versa coincide with ours, and the structural equivalence between $\rfalp$s and $\setaf$s has also been identified there. However, their focus differs from ours. While their work establishes the equivalence between stable models and stable extensions, it does not explore equivalences involving labelling-based semantics or address the controversy relating semi-stable semantics and $L$-stable semantics, which is a key motivation for this work. In comparison with König et al.'s work, the novelty of our proposal lies essentially in the aspects below:

\begin{itemize}
\item Our proofs of these results follow a significantly distinct path as they are based on properties of argument labellings and are deeply rooted in works such as \citep{caminada15equivalence,alcantara2019equivalence,alcantara2021equivalence}.
\item We prove the equivalence between partial stable, well-founded,
regular, stable, and semi-stable model semantics for $\nlp$s and respectively complete, grounded, preferred, stable, and semi-stable labellings for $\setaf$s. In particular, for the first time, an equivalence between $L$-stable model semantics for $\nlp$s and semi-stable labellings for $\setaf$s is established.  
\item We provide a more in-depth analysis of the relationship between $\nlp$s and $\setaf$s. Going beyond just proving semantic equivalence,  we define functions that map labellings to interpretations, and interpretations to labellings. These functions allow us to see interpretations and labellings as equivalent entities, further strengthening the connections between $\nlp$s and $\setaf$s. In substance, we demonstrate that the equivalence also holds at the level of interpretations/labellings.
\end{itemize}

The strong connection we establish between interpretations and labellings opens doors for future exploration. This extends the applicability of our equivalence results to novel semantics beyond those investigated here, potentially even encompassing multivalued settings. This holds particular significance for the logic programming community. Well-established concepts from argumentation, such as argument strength \citep{beirlaen2018argument}, can now be translated and investigated within the context of logic programming. This underscores the value of our decision to employ labellings instead of extensions as a more suitable approach to bridge the gap between $\nlp$s and $\setaf$s. 

Our research offers another key contribution, particularly relevant to the logic programming community: it explores the expressiveness of $\rfalp$s. We demonstrate that a specific combination (denoted by $\mapsto_\utpm$) of program transformations can transform any $\nlp$ into an $\rfalp$ with exactly the same semantics. In simpler terms, $\rfalp$s possess the same level of expressiveness as $\nlp$s. Although each program transformation in $\mapsto_\utpm$ was proposed by Brass and Dix \citep{brass1994disjunctive,brass1997characterizations,brass1999semantics}, the combination of these program transformations (to our knowledge) has not been investigated yet. Then we establish several properties of $\mapsto_\utpm$. Amongst other original contributions of our work related to $\mapsto_\utpm$, we highlight the following results: 

\begin{itemize}
\item Given an $\nlp$, if repeatedly applying $\mapsto_\utpm$ leads to a program where no further transformations are applicable (irreducible program), then the resulting program is guaranteed to be an $\rfalp$.

\item We show that $\mapsto_\utpm$ is confluent, i.e., given an $\nlp$, it does not matter the order by which we apply repeatedly these program transformations, whenever we arrive at an irreducible program, they will always result in the same $\rfalp$ (and in the same corresponding $\setaf$). Hence, besides $\nlp$s and $\rfalp$s being equally expressive, each $\nlp$ is associated with a unique $\rfalp$.
    
\item The $\setaf$ corresponding to an $\nlp$ is invariant with respect to $\mapsto_\utpm$, i.e., if $P_2$ is obtained from $P_1$ via $\mapsto_\utpm$ (denoted by $P_1 \mapsto_\utpm P_2$), both $P_1$ and $P_2$ will be translated into the same $\setaf$.

\item We show that $\mapsto_\utpm$ preserves the semantics for $\nlp$s studied in this paper: if $P_1 \mapsto_\utpm P_2$, then $P_1$ and $P_2$ have the same partial stable models, well-founded models, regular models, stable models, and $L$-stable models.
\end{itemize}

The structure of the paper unfolds as follows: in Section \ref{s:preliminaries}, we establish the fundamental definitions related to $\setaf$s and $\nlp$s. In Section \ref{s:nlp-setaf}, we adapt the procedure from \citep{caminada15equivalence,alcantara2019equivalence} to translate $\nlp$s into $\setaf$s, and subsequently, in the following section, we perform the reverse translation from $\setaf$s to $\nlp$s. In both directions, we demonstrate that our labelling-based approach effectively preserves semantic correspondences, including the challenging case involving the equivalence between semi-stable semantics (on the $\setaf$s side) and $L$-stable semantics (on the $\nlp$s side). In Section \ref{s:slp-setaf}, we focus on $\rfalp$s and reveal that, when restricted to them, the translation processes between $\nlp$s and $\setaf$s are each other's inverse. Then, in Section \ref{s:expressiveness}, we guarantee that $\rfalp$s are as expressive as $\nlp$s. We conclude the paper with a discussion of our findings and outline potential avenues for future research endeavours.

The proofs for all novel results are presented in \ref{s:proofs}.

\section{Preliminaries}\label{s:preliminaries}

\subsection{\setaf}\label{ss:setaf}

In \citep{nielsen2006generalization}, an extension of Dung’s Abstract Argumentation Frameworks ($\aaf$s) \citep{dung1995acceptability} to allow joint attacks on arguments was proposed.  The resulting framework, called $\setaf$, is defined next:

\begin{mdef}[$\setaf$ \protect \citep{nielsen2006generalization}]\label{d:setaf} A framework with sets of attacking arguments ($\setaf$ for short) is a pair $\af = (\ar, \att)$, in which $\ar$ is a finite set of arguments and $\att \subseteq ( 2^\ar \setminus \set{\emptyset}) \times \ar$ is an attack relation such that if $(\mathcal B, a) \in \att$, there is no $\mathcal B' \subset \mathcal B$ such that $(\mathcal B', a) \in \att$, i.e., $\mathcal B$ is a minimal set (w.r.t. $\subseteq$) attacking $a$\footnote{In the original definition of $\setaf$s in \citep{nielsen2006generalization}, attacks are not necessarily subset-minimal.}. By $\att(a) = \set{\mathcal B \subseteq \ar \mid (\mathcal B, a) \in \att}$, we mean the set of attackers of $a$. 
\end{mdef}

In $\aaf$s, only individual arguments can attack arguments. In $\setaf$s, the novelty is that sets of two or more arguments can also attack arguments. This means that $\setaf$s $(\ar, \att)$ with $|\mathcal B| = 1$ for each $(\mathcal B, a) \in \att$ amount to (standard Dung) $\aaf$s. 

The semantics for $\setaf$s are generalisations of the corresponding semantics for $\aaf$s \citep{nielsen2006generalization} and can be defined equivalently in terms of extensions or labellings \citep{flouris2019comprehensive}. Our focus here will be on their labelling-based semantics.

\begin{mdef}[Labellings \protect \citep{flouris2019comprehensive}]\label{d:labelling} Let $\af = (\af, \att)$ be a $\setaf$. A labelling is a function $\nlab : \ar \to \set{\inn, \out, \undec}$. It is \emph{admissible} iff for each $a \in \ar$,

\begin{itemize}
    \item If $\nlab(a) = \inn$, then for each $\mathcal B \in \att(a)$, it holds $\nlab(b) = \out$ for some $b \in \mathcal B$.
    \item If $\nlab(a) = \out$, then there exists $\mathcal B \in \att(a)$ such that $\nlab(b) = \inn$ for each $b \in \mathcal B$.
\end{itemize}

A labelling $\nlab$ is called \emph{complete} iff it is admissible and for each $a \in \ar$,
\begin{itemize}
    \item If $\nlab(a) = \undec$, then there exists $\mathcal B \in \att(a)$ such that $\nlab(b) \neq \out$ for each $b \in \mathcal B$, and for each $\mathcal B \in \att(a)$, it holds $\nlab(b) \neq \inn$ for some $b \in \mathcal B$.
\end{itemize}

\end{mdef}

We write $\inn(\nlab)$ for $\set{a \in \ar \mid \nlab(a) = \inn}$, $\out(\nlab)$ for $\{a \in \ar \mid$ $\nlab(a) = \out \}$, and $\undec(\nlab)$ for $\set{a \in \ar \mid \nlab(a) = \undec}$. As a labelling essentially defines a partition among the arguments, we sometimes write $\nlab$ as a triple $(\inn(\nlab), \out(\nlab), \undec(\nlab))$. Intuitively, an argument labelled $\inn$ represents explicit acceptance; an argument labelled $\out$ indicates rejection; and one labelled $\undec$ is undecided, i.e., it is neither accepted nor rejected. We can now describe the $\setaf$ semantics studied in this paper:

\begin{mdef}[Semantics \citep{flouris2019comprehensive}]\label{d:setaf-semantics} Let $\af = (\ar, \att)$ be a $\setaf$. A complete labelling $\nlab$ is called

\begin{itemize}
    \item \emph{grounded} iff $\inn(\nlab)$ is minimal (w.r.t. $\subseteq$) among all complete labellings of $\af$.
    \item \emph{preferred} iff $\inn(\nlab)$ is maximal (w.r.t. $\subseteq$) among all complete labellings of $\af$.
    \item \emph{stable} iff $\undec(\nlab) = \emptyset$.
    \item \emph{semi-stable} iff $\undec(\nlab)$ is minimal (w.r.t. $\subseteq$) among all complete labellings of $\af$.
\end{itemize}

\end{mdef}

Let us consider the following example:

\begin{example}\label{ex:setaf}
Consider the $\setaf\ \af = (\ar,\att)$ below:

\begin{figure}[ht!]
\centering
\begin{tikzpicture}[>=stealth',shorten >=1pt,auto,node distance=2cm,
  thick,main node/.style={circle,fill=white!20,draw}, inner sep=3pt]

  \node[main node] (c) {$c$};
  \node[main node] (a) [xshift = 1.5cm, below of=c] {$a$};
  \node[main node] (b) [right of=a] {$b$};
  \node[main node] (e) [above of=b] {$e$};
  \node[main node] (d) [xshift = -1.5cm, below of=c] {$d$};

  \path[every node/.style={font=\sffamily\small}]
    (a) edge [->,bend right]	node {} (b)
    (b) edge [->,bend right]	node {} (a)
        edge [->]               node {} (e)
    (c) edge [->,loop right]	node {} (c)
    (d) edge [->,loop left]	node {} (d)
    (e) edge [->,loop left]	node {} (e);
    
    \connectThree[
          @edge 3=->
        ]{d}{a}{c}
      
\end{tikzpicture}
\caption{A $\setaf$\@ $\af$}
\label{f:nadf-setaf}
\end{figure}

Concerning the semantics of $\af$, we have

\begin{itemize}
    \item Complete labellings: $\nlab_1 = (\emptyset, \emptyset, \set{a,b,c,d,e})$, $\nlab_2 =$ $(\set{a}, \set{b}, \set{c,d,e})$ and $\nlab_3 =  (\set{b}, \set{a,e}, \set{c,d})$;
    \item Grounded labellings: $\nlab_1 = (\emptyset, \emptyset, \set{a,b,c,d,e})$;
    \item Preferred labellings: $\nlab_2 =$ $(\set{a}, \set{b}, \set{c,d,e})$ and $\nlab_3 =  (\set{b}, \set{a,e}, \set{c,d})$;
    \item Stable labellings: none;
    \item Semi-stable labellings: $\nlab_3 =  (\set{b}, \set{a,e}, \set{c,d})$.
\end{itemize}
\end{example}

\subsection{Logic Programs and Semantics}

Now, we take a look at propositional Normal Logic Programs. To delve into their definition and semantics, we will follow the presentation outlined in \citep{caminada15equivalence}, which draws from the foundation laid out in \citep{przymusinski90well-founded}.

\begin{mdef}[\protect\citep{caminada15equivalence}] \label{def-lp}
A \emph{rule} $r$ is an expression 
\begin{align}\label{eq:rule}
r: c \leftarrow a_1, \ldots, a_m, \naf b_1, \ldots, \naf b_n
\end{align}
where ($m, n \geq 0$); $c$, each $a_i$ ($1 \leq i \leq m$) and each $b_j$ ($1 \leq j \leq n$) are atoms, and $\mathtt{not}$ represents negation as failure. A literal is either an atom $a$ (positive literal) or a negated atom $\naf a$ (negative literal). Given a rule $r$ as above, $c$ is called the \emph{head} of $r$, which we denote as $\hd(r)$, and $\body(r) = \set{a_1, \ldots, a_m, \naf b_1, \ldots, \naf b_n}$ is called the \emph{body} of $r$. Further, we divide $\body(r)$ into two sets $\body^+(r) = \set{a_1, \ldots, a_m}$ and $\body^-(r) = \set{\naf b_1, \ldots, \naf b_n}$. A \emph{fact} is a rule where $m = n = 0$. A Normal Logic Program ($\nlp$) or simply a \emph{program} $P$ is a finite set of rules. If every $r \in P$ has $body^{-}(r) = \emptyset$, $P$ is a positive program. The \emph{Herbrand Base} of $P$ is the set $\HB_{P}$ of all atoms appearing in $P$.
\end{mdef}

A wide range of $\nlp$ semantics are based on the 3-valued interpretations of programs \citep{przymusinski90well-founded}:

\begin{mdef}[3-Valued Herbrand Interpretation \protect\citep{przymusinski90well-founded}]\label{d:interpretations} 
A 3-valued Herbrand Interpretation $\inter$ (or simply interpretation) of an $\nlp$ $P$ is a pair $\pair{T, F}$ with $T, F \subseteq \HB_P$ and $T \cap F = \emptyset$. The atoms in $T$ are \emph{true} in $\inter$, the atoms in $F$ are \emph{false} in $\inter$, and the atoms in $\HB_P \setminus (T \cup F)$ are \emph{undefined} in $\inter$. For convenience, when the $\nlp$ $P$ is clear from the context, we will refer to the set of undefined atoms in $\HB_P \setminus (T \cup F)$ simply as $\overline{T \cup F}$. 
\end{mdef}

Now we will consider the main semantics for $\nlp$s. Let $\inter = \pair{T, F}$ be a 3-valued Herbrand interpretation of an $\nlp$ $P$; the reduct of $P$ with respect to $\inter$ (written as $P/\inter$) is the $\nlp$ constructed using the following steps:

\begin{enumerate}
  \item Remove any $a \leftarrow a_1, \ldots, a_m,$ $\naf b_1, \ldots, \naf b_n \in P$ such that $b_j \in T$ for some $j$ ($1 \leq j \leq n$);
  \item Afterwards, remove any occurrence of $\naf b_j$ from $P$ such that $b_j \in F$;
  \item Then, replace any occurrence of $\naf b_j$ left by a special atom $\mathbf{u}$ ($\mathbf{u} \not\in \HB_P$).
\end{enumerate}

In the above procedure, $\mathbf{u}$ is assumed to be an atom not in $\HB_{P}$ which is undefined in all interpretations of $P$ (a constant). Note that $P/\inter$ is a positive program since all negative literals have been removed. As a consequence, $P/\inter$ has a unique least 3-valued model \citep{przymusinski90well-founded}, obtained by the $\Psi$ operator:

\begin{mdef}[$\Psi$ Operator \protect\citep{przymusinski90well-founded}]
Let $P$ be a positive program and $\mathcal J =  \pair{T, F}$ be an interpretation. Define $\Psi_P(\mathcal J) = \pair{T', F'}$, where
\begin{itemize}
  \item $c \in T'$ iff $c \in HB_P$ and there exists $c \leftarrow a_1, \ldots, a_m \in P$ such that for all $i$, $1 \leq i \leq m$, $a_i \in T$;
  \item $c \in F'$ iff $c \in HB_P$ and for every $c \leftarrow a_1, \ldots, a_m \in P$, there exists $i$, $1 \leq i \leq m$, such that $a_i \in F$.
\end{itemize}

The least 3-valued model of $P$ is given by $ \Psi^{\uparrow\ \omega}_P$ \citep{przymusinski90well-founded}, the least fixed point of $\Psi_P$ iteratively obtained as follows:
\begin{align*}
    \Psi^{\uparrow\ 0}_P = & \pair{\emptyset, \HB_P}\\ 
    \Psi^{\uparrow\ i + 1}_P = & \Psi_P(\Psi^{\uparrow\ i}_P)\\
    \Psi^{\uparrow\ \omega}_P = & \pair{\bigcup_{i < \omega} \set{T_i \mid \Psi^{\uparrow\ i}_P = \pair{T_i, F_i}}, \bigcap_{i < \omega} \set{F_i \mid \Psi^{\uparrow\ i}_P = \pair{T_i, F_i}}}
\end{align*}
where $\omega$ denotes the first infinite ordinal.
\end{mdef}

We can now describe the logic programming semantics studied in this paper:

\begin{mdef} \label{d:nlp-semantics} Let $P$ be an $\nlp$ and $\inter = \pair{T, F}$ be an interpretation; by $\Omega_P(\inter) = \Psi^{\uparrow\ \omega}_\frac{P}{\inter}$, we mean the least 3-valued model of $\frac{P}{\inter}$. We say that

\begin{itemize}
  \item $\inter$ is a partial stable model of $P$ iff 
	$\Omega_P(\inter) = \inter$ \citep{przymusinski90well-founded}.
  \item $\inter$ is a well-founded model of $P$ iff 
	$\inter$ is a partial stable model of $P$ where there is no partial stable model $\inter ' = \pair{T', F'}$ of $P$ such that $T' \subset T$, i.e., $T$ is minimal (w.r.t.\ set 
	inclusion) among all partial stable models of $P$ \citep{przymusinski90well-founded}.
  \item $\inter$ is a regular model of $P$ iff $\inter$ is a partial stable model of $P$ where there is no partial stable model $\inter ' = \pair{T', F'}$ of $P$ such that $T \subset T'$, i.e., $T$ is maximal (w.r.t.\ set inclusion) among all partial stable models of $P$ \citep{eiter97partial}.
  \item $\inter$ is a (2-valued) stable model of $P$ iff
	$\inter$ is a partial stable model of $P$ where $T \cup F = \HB_P$ \citep{przymusinski90well-founded}.
  \item $\inter$ is an $L$-stable model of $P$ iff $\inter$ is a partial stable model of $P$ where there is no partial stable model $\inter ' = \pair{T', F'}$ of $P$ such that $T \cup F \subset T' \cup F'$, i.e., $T \cup F$ is maximal (w.r.t.\ set
	inclusion) among all partial stable models of $P$ \citep{eiter97partial}.
\end{itemize}

\end{mdef}

Although some of these definitions are not standard in logic programming literature, their equivalence is proved in \citep{caminada15equivalence}. This format helps us to relate $\nlp$ and $\setaf$ semantics due to the structural similarities between Definition \ref{d:nlp-semantics} and Definitions \ref{d:labelling} and \ref{d:setaf-semantics}. We illustrate these semantics in the following example:

\begin{example}\label{ex:nlp}
Consider the following logic program $P$:
$$
\begin{array}{llcllcll}
r_1:	&	a \leftarrow \mathtt{not}\; b 					& &
r_2:	&	b \leftarrow \mathtt{not}\; a & & r_3:  & c \leftarrow \mathtt{not}\; a, \mathtt{not}\; c\\
r_4:	&	c \leftarrow \mathtt{not}\; c, \mathtt{not}\; d					& &
r_5:	& d \leftarrow \mathtt{not}\; d	 & & 
r_6:	& e \leftarrow \mathtt{not}\; b, \mathtt{not}\; e
\end{array}
$$

\noindent This program has 

\begin{itemize}
\item Partial Stable Models:
$\mathcal M_1 = \langle \emptyset, \emptyset \rangle$, $\mathcal M_2 = \langle \set{ a }, \set{ b } \rangle$ and $\mathcal M_3 = \langle \set{ b }, \set{ a, e } \rangle;$
\item Well-founded model: $\mathcal M_1 = \langle \emptyset, \emptyset \rangle$;
\item Regular models: $\mathcal M_2 = \langle \set{ a }, \set{ b } \rangle$ and $\mathcal M_3 = \langle \set{ b }, \set{ a, e } \rangle$;
\item Stable models: none; 
\item $L$-stable model: $\mathcal M_3 = \langle \set{ b }, \set{ a, e } \rangle$. 
\end{itemize}
\end{example}

\section{From $\nlp$ to $\setaf$}
	\label{s:nlp-setaf}

In this section, we revisit the three-step process of argumentation framework instantiation as employed in \citep{caminada15equivalence} for translating an $\nlp$ into an $\aaf$. This method is based on
the approach introduced by \citep{wu2009complete} and shares similarities with the procedures used in ASPIC \citep{caminada2005axiomatic,caminada2007evaluation} and logic-based argumentation \citep{gorogiannis2011instantiating}. Its first step involves taking an $\nlp$ and constructing its associated $\aaf$. Then, we apply $\aaf$ semantics in the second step, followed by an analysis of the implications of these semantics at the level of conclusions (step 3). In our case, starting with an $\nlp$ $P$, we derive the associated $\setaf$ $(\ar_P, \att_P)$. Unlike the construction described in \citep{caminada15equivalence}, rules with identical conclusions in $P$ will result in a single argument in $\ar_P$. This distinction is capital for establishing the equivalence results between $\nlp$s and $\setaf$s. Additionally, it simplifies steps 2 and 3, making them more straightforward to follow. We now detail this process.

\subsection{$\setaf$ Construction} \label{subsec-step1}

We will devise one translation from $\nlp$ to $\setaf$ that is sufficiently robust to guarantee the equivalence between various kinds of $\nlp$s models and $\setaf$s labellings. Specifically, our approach will establish the correspondence between partial stable models and complete labellings, well-founded models and grounded labellings, regular models and preferred labellings, stable models and stable labellings, $L$-stable models and semi-stable labellings. Our method is built upon a translation from $\nlp$ to $\aaf$ proposed in \citep{caminada15equivalence}, where $\nlp$ rules are directly translated into arguments. We will adapt this approach for $\setaf$ by employing the translation method outlined in \citep{caminada15equivalence} to construct statements, and then statements corresponding to rules with the same head will be grouped to form a single argument. Taking an $\nlp$ $P$, we can start to construct statements recursively as follows:

\begin{mdef}[Statements and Arguments] \label{d:argument}
Let $P$ be an $\nlp$.

\begin{itemize}
  \item If $c \leftarrow \naf b_1, \ldots, \naf b_n$
	is a rule in $P$, then it is also a statement (say $s$) with
	\begin{itemize}
	  \item	$\Conc(s) = c$,
	  \item $\Rules(s) = \{ c \leftarrow \naf b_1, \ldots, 
		\naf b_n \}$,
	  \item $\Vul(s) = \{ b_1, \ldots, b_n \}$, and
	  \item $\Sub(s) = \{ s \}$.
	\end{itemize}
  \item If $c \leftarrow a_1,\ldots,a_m, \naf b_1,\ldots,
	\naf b_n$ is a rule in $P$ and for each $a_i$ ($1 \leq i 
	\leq m$) there exists a statement $s_i$ with $\Conc(s_i) = a_i$
	and $c \leftarrow a_1,\ldots, a_m, \naf b_1,\ldots,
	\naf b_n$ is not contained in $\Rules(s_i)$,
	then $c \leftarrow (s_1),\ldots,(s_m),$ $\naf b_1,\ldots,
	\naf b_n$ is a statement (say $s$) with
	\begin{itemize}
	  \item $\Conc(s) = c$,
	  \item
		  $\Rules(s) = \Rules(s_1) \cup \ldots \cup \Rules(s_m)\ \cup
		  \{ c \leftarrow a_1,\ldots, a_m, \naf b_1,\ldots, \naf b_n \}$	  
	  \item $\Vul(s) = \Vul(s_1)$ $\cup \ldots \cup$ $\Vul(s_m)$ $\cup$ 
		$\{b_1, \ldots, b_n \}$, and
	  \item $\Sub(s) = \{ s \} \cup \Sub(s_1) \cup \ldots \cup \Sub(s_m)$.
	\end{itemize}
\end{itemize}

By $\mathfrak S_P$ we mean the set of all statements we can construct from $P$ as above. Then we define $\ar_P = \set{\Conc(s) \mid s \in \mathfrak S_P}$ as the set of all arguments we can construct from $P$. For an argument $c$ from $P$ ($c \in \ar_P$), we have that 

\begin{itemize}
	  \item	$\Conc(c) = c$,
	  \item $\Vul_P(c) = \set{\Vul(s) \mid s \in \mathfrak S_P \textit{ and } \Conc(s) = c}$, and
	\end{itemize}
\end{mdef}

If $c$ is an argument, then $\Conc(c)$ is referred to as the \emph{conclusion} of $c$, and $\Vul_P(c)$ is referred to as the \emph{vulnerabilities} of $c$ in $P$. When the context is clear, we will write simply $\Vul(c)$ instead of $\Vul_P(c)$.

Now we will clarify the connection between the existence of statements and the existence of a derivation in a reduct. 

\begin{restatable}{lem}{statementreduct}\label{l:statement-reduct} Let $P$ be an $\nlp$, $\inter = \pair{T, F}$ an interpretation and $\Omega_P(\inter) = \pair{T', F'}$ the least 3-valued model of $\frac{P}{\inter}$. It holds

\begin{enumerate}[label=(\roman*)]
     \item\label{i:stat-redI} $c \in T'$ iff there exists a statement $s$ constructed from $P$ such that $\Conc(s) = c$ and $\Vul(s) \subseteq F$.
    \item\label{i:stat-redII} $c \in F'$ iff for every statement $s$ constructed from $P$ such that $\Conc(s) = c$, we have $\Vul(s) \cap T \neq \emptyset$  
\end{enumerate} 
\end{restatable}

We can prove both results in Lemma \ref{l:statement-reduct} by induction. Assuming that $\Psi^{\uparrow\ i}_\frac{P}{\inter} = \pair{T_i, F_i}$ for each $i \in \mathbb{N}$, we can prove the right-hand side of item \ref{i:stat-redI} and the left-hand side of item \ref{i:stat-redII} by induction on the value of $i$ after guaranteeing the following results:

\begin{itemize}
\item If $c \in T_i$, then there exists a statement $s$ constructed from $P$ such that $\Conc(s) = c$ and $\Vul(s) \subseteq F$.
\item If $c \not\in F_i$, then there exists a statement $s$ constructed from $P$ such that $\Conc(s) = c$ and $\Vul(s) \cap T = \emptyset$.  
\end{itemize}

The remaining cases of Lemma \ref{l:statement-reduct} can be proved by structural induction on the construction of a statement $s$ (see a detailed account of the proof of Lemma \ref{l:statement-reduct} in Section \ref{ss:nlp-setaf-proofs} of \ref{s:proofs}). 

Lemma \ref{l:statement-reduct} ensures that statements are closely related to derivations in a reduct. An atom $c$ is true in the least 3-valued model of $\frac{P}{\inter}$ iff we can construct a statement with conclusion $c$ and whose vulnerabilities are false according to $\inter$; otherwise, $c$ is false in the least 3-valued model of $\frac{P}{\inter}$ iff for every statement whose conclusion is $c$, at least one of its vulnerabilities is true in $\inter$. The next result is a direct consequence of Lemma \ref{l:statement-reduct}:

\begin{cor}
\label{c:statement-reduct}  Let $P$ be an $\nlp$.

\begin{itemize}
\item Assume $\inter = \pair{\emptyset, \HB_P}$ and $\Omega_P(\inter) = \pair{T', F'}$. It holds that $c \in T'$ iff there exists a statement $s$ constructed from $P$ such that $\Conc(s) = c$.
\item There is no statement $s$ constructed from $P$ such that $\Conc(s) = c$ iff $c \in F'$ for every interpretation $\inter$ with $\Omega_P(\inter) = \pair{T', F'}$. 
\end{itemize}
\end{cor}

The reduct of $P$ with respect to $\pair{\emptyset, \HB_P}$ gives us all the possible derivations of $P$, and from these derivations, we can construct all the statements associated with $P$. On the other hand, the atoms that are lost in the translation, i.e., the atoms not associated with statements are simply those that are false in the least 3-valued model of every possible reduct of $P$. Besides establishing connections between statements and derivations in a reduct, Lemma \ref{l:statement-reduct} also plays a central role in the proof of Theorems \ref{t:modarg} and \ref{t:psmcomp}.  

Apart from that, intuitively, we can see a statement as a tree-like structure representing a possible derivation of an atom from the rules of a program. In contrast, an argument for $c$ in $P$ is associated with the (derivable) atom $c$ itself and can be obtained by collecting all the statements with the same conclusion $c$ (i.e., all the possible ways of deriving $c$ in $P$). 

\begin{example}\label{ex:argument}
Consider the $\nlp$ $P$ below with rules $\set{r_1, \ldots, r_8}$:
$$
\begin{array}{llcllcll}
r_1:	&	a			& &
r_2:	&	b \leftarrow a	    & & 
r_3:	&	c \leftarrow \mathtt{not}\; c				\\
r_4:	&	d \leftarrow b, \mathtt{not}\; a, \mathtt{not}\; d		
       & & 
r_5:	&	d \leftarrow \mathtt{not}\; c, \mathtt{not}\; d   &  & 
r_6:	&	e \leftarrow b, c, \mathtt{not}\; e			\\
r_7:	&	c \leftarrow f, \mathtt{not}\; g		& & 
r_8:	&	f \leftarrow c, g & & \ & \ 	
\end{array}
$$

According to Definition \ref{d:argument}, we  can construct the following statements from $P$:
$$
\begin{array}{llcllcll}
s_1:	&	a 					& &
s_2:	&	b \leftarrow (s_1) & & s_3:  & c \leftarrow \mathtt{not}\; c\\
s_4:	&	d \leftarrow (s_2), \mathtt{not}\; a, \mathtt{not}\; d					& &
s_5:	& d \leftarrow \mathtt{not}\; c, \mathtt{not}\; d	 & & 
s_6:	& e \leftarrow (s_2), (s_3),\mathtt{not}\; e
\end{array}
$$

In the next table, we give the conclusions and vulnerabilities of each statement:
\begin{center}
\renewcommand{\arraystretch}{1.2}
\begin{tabular}{ c | *{6}{c}}
 & $s_1$ & $s_2$ & $s_3$ & $s_4$ & $s_5$ & $s_6$ \\
 \hline
$\Conc(.)$ & $a$ & $b$ & $c$ & $d$ & $d$ & $e$\\
$\Vul(.)$ & $\emptyset$ & $\emptyset$ & $\{ c \}$ & $\{ a, d \}$ & $\{ c, d \}$ & $\{ c, e \}$
\end{tabular}
 \renewcommand{\arraystretch}{1}
\end{center}

Alternatively, we can depict statements as possible derivations as in Fig. \ref{f:statement-derivation}:

\begin{figure}[h!]
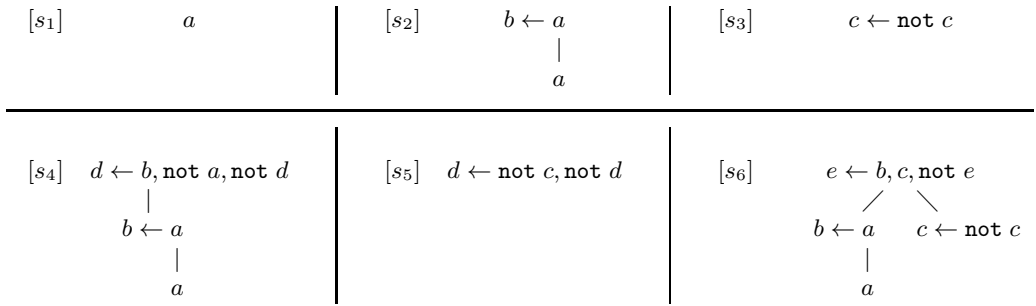

$$
\begin{array}{lcl|llcl|llc}
\ [s_1] & a 			& \ & \ & [s_2] & b \leftarrow a 			& \ & \ & [s_3] & c \leftarrow \mathtt{not}\; c\\
 \ 	  & \ 			 		& \ & \ & \ 	  & \ \ \ \ \ \ |	 	 	 	 	& \ & \ & \ 	   &  \ 	\\
 \ 	  & \ 					& \ & \ & \ 	  & \ \ \ \ \ \ a   			& \ & \ & \ 	   & \	\\
\hline
 \ 	  & \ 			 		& \ & \ & \ 	  & \ 		 	 	 	 & \ & \ & \ 	   & \ 			\\
\ [s_4]  & d \leftarrow b, \mathtt{not}\; a, \mathtt{not}\; d	 & \ & \ & [s_5] & d \leftarrow \mathtt{not}\; c, \mathtt{not}\; d	 & \ & \ & [s_6] & e \leftarrow b, c, \mathtt{not}\; e	\\
\ 	  &  | \ \ \ \ \ \ \ \ \ \ 			 		& \ & \ & \ 	  & \ 		 	 	 	 & \ & \ & \ 	   & \diagup   \ \ \ \	\diagdown	\\ 
 \ 	  & b \leftarrow a \ \ \ \ \ \ \ \ \ 			 		& \ & \ & \ 	  & \ 			 	 	 & \ & \ & \ 	   &\ \ \ \ b \leftarrow a \ \ \ \ \ c \leftarrow \mathtt{not}\; c	\\
\ 	  & | \ \ \  			 		& \ & \ & \ 	  & \ 		 	 	 	 & \ & \ & \ 	   & | \ \ \ \ \ \ \ \		\\
 \ 	  & a \ \ \   			 		& \ & \ & \ 	  & \ 			 	 	 & \ & \ & \ 	   & a \ \ \ \ \ \ \ \
\end{array}
$$
\caption{Statements constructed from $P$}
\label{f:statement-derivation}
\end{figure}

The vulnerabilities of a statement $s$ are associated with the negative literals found in the derivation of $s$. If $\mathtt{not}\; a$ is one of them, we know that $a$ is one of its vulnerabilities. This means that if $a$ is derived, then $\Conc(s)$ cannot be obtained via this derivation represented by $s$. However, it can still be obtained via other derivations/statements. For instance, in the program $P$ of Example \ref{ex:argument}, the derivation of $a$ suffices to prevent the derivation of $d$ via statement $s_4$ (for that reason, $a \in \Vul(s_4)$), but we still can derive $d$ via $s_5$. Notice also that there are no statements with conclusions $f$ and $g$. From Corollary \ref{c:statement-reduct}, we know that it is not possible to derive them in $P$ as they are false in the least 3-valued model of each reduct of $P$. In addition, to determine the vulnerabilities of an atom (and not only of a specific derivation leading to this atom), we collect these data about the statements with the same conclusions to give the conclusions and vulnerabilities of each argument. In our example, we obtain the following results:
\begin{center}
\renewcommand{\arraystretch}{1.2}
\begin{tabular}{ c | *{5}{c}}
 & $a$ & $b$ & $c$ & $d$ & $e$ \\
 \hline
$\Conc(.)$ & $ a $ & $ b$ & $c$ & $d$ & $e$ \\
$\Vul(.)$ & $\set{ \emptyset }$ & $\set{ \emptyset }$ & $\set{ c }$ &  $\set{ \set{a, d},\set{c, d} }$     & $\set{\set{c, e}}$
\end{tabular}
 \renewcommand{\arraystretch}{1}
\end{center}

\end{example}

As the vulnerabilities of an atom/argument $a$ are a collection of the vulnerabilities of the statements whose conclusion is $a$, any set containing at least one atom in each of these statements suffices to prevent the derivation of $a$ in $P$. In our example, there are two statements with the same conclusion $d$ and $\Vul(d) = \set{\set{a, d}, \set{c, d}}$. Thus any set of atoms containing $\set{d}$ or $\set{a, c}$ prevents the conclusion of $d$ in $P$. We will resort to these minimal sets to determine the attack relation:

\begin{mdef} \label{d:attack}
Let $P$ be an $\nlp$ and let $\mathcal B$ and $a$ be respectively a set of arguments and an argument in the sense of Definition \ref{d:argument}. We say that $(\mathcal B, a) \in \att_P$ iff $\mathcal B$ is a minimal set (w.r.t. set inclusion) such that for each $V \in \Vul_P(a)$, there exists $b \in \mathcal B \cap V$.
\end{mdef}

For the arguments of Example \ref{ex:argument}, it holds that both $a$ and $b$ are not attacked, $c$ attacks itself, $c$ attacks $e$, $e$ attacks itself, $d$ attacks itself, $a$ and $c$ (collectively) attack $d$. This strategy of extracting statements from $\nlp$s rules and then gathering those with identical conclusions into arguments is not novel; in \citep{alcantara2019equivalence}, the authors proposed a translation from $\nlp$s into Abstract Dialectical Frameworks \citep{brewka2013abstract,brewka2010abstract} by following a similar path. Using the thus-defined notions of arguments and attacks, we define the $\setaf$ associated with an $\nlp$.

\begin{mdef} \label{d:nlp-setaf}
Let $P$ be an $\nlp$. We define its associated $\setaf$
as $\af_P = (\ar_P, \att_P)$, where $\ar_P$ is the set of arguments in the sense of Definition \ref{d:argument} and $\att_P$ is the attack relation in the sense of Definition \ref{d:attack}.
\end{mdef}

As an example, the $\setaf$ $\af_P = (\ar_P,\att_P)$ associated with the $\nlp$ of Example \ref{ex:argument} is depicted in Fig. \ref{f:nlp-setaf}.

\begin{figure}
\centering
\begin{tikzpicture}[>=stealth',shorten >=1pt,auto,node distance=2cm,
  thick,main node/.style={circle,fill=white!20,draw}, inner sep=3pt]

  \node[main node] (d) {$d$};
  \node[main node] (c) [xshift = 1.5cm, below of=d] {$c$};
  \node[main node] (e) [right of=c] {$e$};
  \node[main node] (b) [above of=e] {$b$};
  \node[main node] (a) [xshift = -1.5cm, below of=d] {$a$};

  \path[every node/.style={font=\sffamily\small}]
    (c) edge [->] node {} (e)
        edge [->,loop above]	node {} (c)
    (d) edge [->,loop left]	node {} (d)
    (e) edge [->,loop right]	node {} (e);
    
    \connectThree[
          @edge 3=->
        ]{a}{c}{d}
      
\end{tikzpicture}
\caption{A $\setaf$\@ $\af_P = (\ar_P,\att_P)$}
\label{f:nlp-setaf}
\end{figure}

\subsection{Equivalence Results}

Once the $\setaf$ has been constructed, we show the equivalence between the semantics for an $\nlp$ $P$ and their counterpart for the associated $\setaf$ $\af_P$. One distinguishing characteristic of our approach in comparison with König et al.'s proposal \citep{konig22just} is that it is more organic. We prove the equivalence results by identifying connections between fundamental notions used in the definition of the semantics for $\nlp$s and $\setaf$s. With this purpose, we introduce two functions: $\LabMod$ associates an interpretation to each labelling while $\ModLab$ associates a labelling to each interpretation. We then investigate the conditions under which they are each other's inverse and employ these results to prove the equivalence between the semantics. These functions essentially permit us to treat interpretations and labellings interchangeably. 

\begin{mdef}[$\LabMod$ and $\ModLab$ Functions] \label{d:labmod}
Let $P$ be an $\nlp$, $\af_P = (\ar_P, \att_P)$ be its associated $\setaf$, $\model$ be the set of all the 3-valued interpretations of $P$ and $\lab$ be the set of all labellings of $\af_p$. We introduce a function $\LabMod : \lab \to \model$ such that $\LabMod(\nlab) = \left<
T , F \right>$, in which

\begin{itemize}
    \item $T = \set{c \in \HB_P \mid c \in \ar_P \textit{ and } \nlab(c) = \inn}$;
    \item $F = \{c \in \HB_P \mid c \not\in \ar_P \textit{ or } c \in \ar_P \textit{ and } \nlab(c) = \out \}$;
    \item $\overline{T \cup F} = \set{c \in \HB_P \mid c \in \ar_P \textit{ and } \nlab(c) = \undec}$.
\end{itemize}

We introduce a function $\ModLab : \model \to \lab$ such that for any $\mathcal I = \left< T, F \right> \in \model$ and any $c \in \ar_P$,

\begin{itemize}
    \item $\ModLab(\mathcal I)(c) = \inn$ if $c\in T$;
    \item $\ModLab(\mathcal I)(c) = \out$ if $c \in F$; 
    \item $\ModLab(\mathcal I)(c) = \undec$ if $c \not\in T \cup F$. 
\end{itemize}

$\ModLab(\mathcal I)(c)$ is not defined if $c \not\in \ar_P$.
\end{mdef}

The correspondence between labellings and interpretations is clear for those atoms $c \in \HB_P$ in which $c \in \ar_P$. In this case, we have that $c$ is interpreted as true iff $c$ is labelled as $\inn$; $c$ is interpreted as false iff $c$ is labelled as $\out$. In contradistinction, those atoms $c \in \HB_P$ not associated with arguments ($c \not\in \ar_P$) are simply interpreted as false. This will suffice to guarantee our results; next theorem assures us that $\ModLab(\LabMod(\mathcal L)) = \mathcal L$:

\begin{restatable}{theo}{inverse}\label{t:inverse}
Let $P$ be an $\nlp$ and $\af_P = (\ar_P, \att_P)$ be the associated $\setaf$. For any labelling $\mathcal L$ of $\af_P$, it holds $\ModLab(\LabMod(\mathcal L)) = \mathcal L$. 
\end{restatable}

In general, $\LabMod(\ModLab(\inter))$ is not equal to $\inter$, because of those atoms $c$ occurring in an $\nlp$ $P$, but not in $\ar_P$. However, when $\mathcal M$ is a partial stable model, $\LabMod(\ModLab(\mathcal M)) = \mathcal M$:

\begin{restatable}{theo}{ModArg}\label{t:modarg}
Let $P$ be an $\nlp$, $\af_P = (\ar_P, \att_P)$ be the associated $\setaf$ and $\mathcal M = \left< T, F \right>$ be a partial stable model of $P$. It holds that $\LabMod(\ModLab(\mathcal M)) = \mathcal M$.
\end{restatable}

This means that when restricted to partial stable models and complete labellings, $\LabMod$ and $\ModLab$ are each other's inverse. From Lemma \ref{l:statement-reduct}, and Theorems \ref{t:inverse} and \ref{t:modarg}, we can obtain the following result:

\begin{restatable}{theo}{psmcomp}\label{t:psmcomp}
Let $P$ be an $\nlp$ and $\af_P = (\ar_P, \att_P)$ be the associated $\setaf$. It holds

\begin{itemize}
    \item $\nlab$ is a complete labelling of $\af_P$ iff $\LabMod(\nlab)$ is a partial stable model of $P$.
    \item $\mathcal M$ is a partial stable model of $P$ iff $\ModLab(\mathcal M)$ is a complete labelling of $\af_P$.
\end{itemize}
\end{restatable}

Theorem \ref{t:psmcomp} is one of the main results of this paper. It plays a central role in ensuring the equivalence between the semantics for $\nlp$ and their counterpart for $\setaf$:

\begin{restatable}{theo}{equivalence}\label{t:equivalence}
Let $P$ be an $\nlp$ and $\af_P = (\ar_P, \att_P)$ be the associated $\setaf$. It holds

\begin{enumerate}
    \item $\nlab$ is a grounded labelling of $\af_P$ iff $\LabMod(\nlab)$ is a well-founded model of $P$.
    \item $\nlab$ is a preferred labelling of $\af_P$ iff $\LabMod(\nlab)$ is a regular model of $P$.
    \item $\nlab$ is a stable labelling of $\af_P$ iff $\LabMod(\nlab)$ is a stable model of $P$.
    \item $\nlab$ is a semi-stable labelling of $\af_P$ iff $\LabMod(\nlab)$ is an $L$-stable model of $P$.
\end{enumerate}

\end{restatable}

The following result is a direct consequence of Theorems \ref{t:modarg} and \ref{t:equivalence}:

\begin{restatable}{cor}{equivalenceII}\label{c:equivalenceII}
Let $P$ be an $\nlp$ and $\af_P = (\ar_P, \att_P)$ be the associated $\setaf$. It holds

\begin{enumerate}
    \item $\mathcal M$ is a well-founded model of $P$ iff $\ModLab(\mathcal M)$ is a grounded labelling of $\af_P$.
    \item $\mathcal M$ is a regular model of $P$ iff $\ModLab(\mathcal M)$ is a preferred labelling of $\af_P$.
    \item $\mathcal M$ is a stable model of $P$ iff $\ModLab(\mathcal M)$ is a stable labelling of $\af_P$.
    \item $\mathcal M$ is an $L$-stable model of $P$ iff $\ModLab(\mathcal M)$ is a semi-stable labelling of $\af_P$.
\end{enumerate}

\end{restatable}

Next, we consider the $\nlp$ exploited by Caminada et al. \citep{caminada15equivalence} as a counterexample to show that in general, $L$-stable models and semi-stable labellings do not coincide with each other in their translation from $\nlp$s to $\aaf$s:

\begin{example}\label{ex:nlp-setaf}
Let $P$ be the $\nlp$ and $\af_P$ be the associated $\setaf$ depicted in Fig. \ref{f:nlp-setafII}:

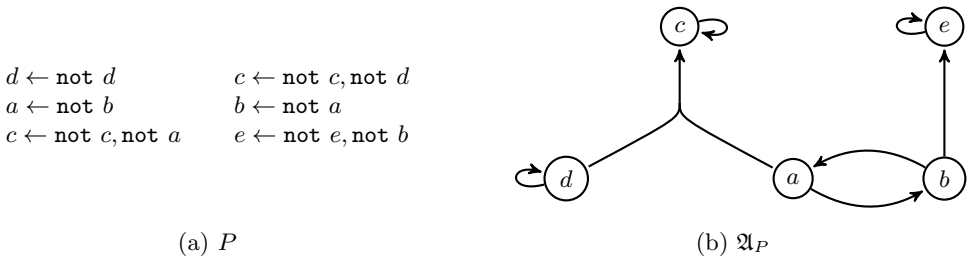
\begin{figure}[ht!]
    \centering
     \begin{subfigure}[b]{0.43\textwidth}
$$
\begin{array}{lcl}
d \leftarrow \naf d	 & & c \leftarrow \naf c, \naf d\\
a \leftarrow \naf b	 & & b \leftarrow \naf a\\
c \leftarrow \naf c, \naf a	& & e \leftarrow \naf e, \naf b
\end{array}
$$
\vspace{1.75em}
\caption{$P$}
     \end{subfigure}
     \hfill
    \begin{subfigure}[b]{0.54\textwidth}
         \centering
\begin{tikzpicture}[>=stealth',shorten >=1pt,auto,node distance=2cm,
  thick,main node/.style={circle,fill=white!20,draw}, inner sep=3pt]

  \node[main node] (c) {$c$};
  \node[main node] (a) [xshift = 1.5cm, below of=c] {$a$};
  \node[main node] (b) [right of=a] {$b$};
  \node[main node] (e) [above of=b] {$e$};
  \node[main node] (d) [xshift = -1.5cm, below of=c] {$d$};

  \path[every node/.style={font=\sffamily\small}]
    (a) edge [->,bend right]	node {} (b)
    (b) edge [->,bend right]	node {} (a)
        edge [->]               node {} (e)
    (c) edge [->,loop right]	node {} (c)
    (d) edge [->,loop left]	node {} (d)
    (e) edge [->,loop left]	node {} (e);
    
    \connectThree[
          @edge 3=->
        ]{d}{a}{c}     
\end{tikzpicture} 
         \caption{$\af_P$}
         \label{f:nlp-setafIIb}
     \end{subfigure}
  \caption{$\nlp$ $P$ and its associated $\setaf$ $\af_P$}
        \label{f:nlp-setafII}   
\end{figure}
\end{example}

As expected from Theorems \ref{t:psmcomp} and \ref{t:equivalence}, we obtain in Table \ref{tab:p-afp} the equivalence between partial stable models and complete labellings, well-founded models and grounded labellings, regular models and preferred labellings, stable models and stable labellings, $L$-stable models and semi-stable labellings. We emphasise the coincidence between $L$-stable models and semi-stable labellings in Table \ref{tab:p-afp} as it does not occur in \citep{caminada15equivalence}. In that reference, the associated $\aaf$ possesses two semi-stable labellings in contrast with the unique $L$-stable model $\mathcal M_3$ of $P$. In the next two sections, we will show that this relation between $\nlp$s and $\setaf$s has even deeper implications.

\begin{table}[!ht]
 \centering
 \caption{Semantics for $P$ and $\af_P$}
 \label{tab:p-afp}
 {\tablefont\begin{tabular}{@{\extracolsep{\fill}}llcl}
   \topline
    \multirow{3}{*}[0pt]{Partial Stable Models } & $\mathcal M_1 = \pair{ \emptyset, \emptyset }$ & \multirow{3}{*}[0pt]{Complete Labellings} & $\nlab_1 = ( \emptyset, \emptyset, \set{ a, b, c, d, e } )$\\
   & $\mathcal M_2 = \pair{ \set{a}, \set{b} }$  & & $\nlab_2 = ( \set{a}, \set{b}, \set{c, d, e} )$\\
   &  $\mathcal M_3 = \pair{\set{b}, \set{a,e} }$ &  & $\nlab_3 = ( \set{ b }, \set{a, e}, \set{c, d} )$ \\
   \hline
    Well-Founded Models\ & $\mathcal M_1 = \pair{\emptyset, \emptyset }$ &   Grounded Labellings & $\nlab_1 = ( \emptyset, \emptyset, \set{ a, b, c, d, e } )$  \\
    \hline
    \multirow{2}{*}[0pt]{Regular Models} & $\mathcal M_2 = \pair{ \set{a}, \set{b} }$  & \multirow{2}{*}[0pt]{Preferred Labellings} & $\nlab_2 = ( \set{ a }, \set{ b }, \set{ c, d, e } )$ \\
    & $\mathcal M_3 = \pair{ \set{b}, \set{a, e} }$   &  &  $\nlab_3 = ( \set{ b }, \set{ a, e}, \set{ c, d } )$ \\
    \hline
    Stable Models &  None &  Stable Labellings &  None \\ 
    \hline
    $L$-stable Models & $\mathcal M_3 = \pair{ \set{b}, \set{a, e} }$ & \multirow{1}{*}[5pt]{Semi-stable} \ & $\nlab_3 = ( \set{ b }, \set{ a,  e }, \set{ c, d } )$ \vspace{-.6em} \\
    & & Labellings &
   \botline
    \end{tabular}}
\end{table} 

\section{From $\setaf$ to $\nlp$}\label{s:setaflp}

Now we will provide a translation in the other direction, i.e., from $\setaf$s to $\nlp$s. As in the previous section, this translation guarantees the equivalence between the semantics for $\nlp$s and their counterpart for $\setaf$s. 

\begin{mdef}\label{d:setaf-nlp} Let $\af = (\ar, \att)$ be a $\setaf$. For any argument $a \in \ar$, we will assume $\mathcal V_a = \{ V \mid V \textit{ is a minimal set (w.r.t. set inclusion) such that  for each } \mathcal B \in \att(a), \textit{ there exists } b \in \mathcal B \cap V \}$.
We define the associated $\nlp$ $P_\af$ as follows:
\[P_\af = \set{a \leftarrow \naf b_1, \ldots \naf b_n \mid a \in \ar, V \in \mathcal V_a \textit{ and } V = \set{b_1, \ldots, b_n}}. \]
\end{mdef}

\begin{example}\label{ex:setaf-nlp}
Recall the $\setaf$ $\af$ of Example \ref{ex:setaf} (it is the same as that in Fig \ref{f:nlp-setafIIb}). The associated $\nlp$ $P_\af$ is 
$$
\begin{array}{lcl}
d \leftarrow \naf d	 & & c \leftarrow \naf c, \naf d\\
a \leftarrow \naf b	 & & b \leftarrow \naf a\\
c \leftarrow \naf c, \naf a	& & e \leftarrow \naf e, \naf b	
\end{array}
$$
\end{example}

Notice that $P_\af$ and the $\nlp$ $P$ of Example \ref{ex:nlp-setaf} are the same. As it will be clear in the next section, this is not merely a coincidence. Besides, from Definition \ref{d:setaf-nlp}, it is clear that $\HB_{P_\af} = \ar$. Consequently, when considering a $\setaf$ $\af$ and its associated $\nlp$ $P_\af$, the definition of the function $\lm$ (resp. $\ml$), which associates labellings with interpretations (resp.  interpretations with labellings), will be simpler than the definition of $\LabMod$ (resp. $\ModLab$) presented in the previous section. 

\begin{mdef}[$\lm$ and $\ml$ Functions] \label{d:lm}
Let $\af$ be a $\setaf$ and $P$ be its associated $\nlp$, $\lab$ be the set of all labellings of $\af$ and $\model$ be the set of all the 3-valued interpretations of $P_\af$. We introduce the functions 

\begin{itemize}
\item $\lm : \lab \to \model$, in which 
\[ \lm(\nlab) = \pair{\inn(\nlab), \out(\nlab)}. \] 
Obviously $\overline{\inn(\nlab) \cup \out(\nlab)} = \undec(\nlab)$.
\item $\ml : \model \to \lab$, in which for $\mathcal M = \pair{T, F} \in \model$, 
\[ \ml(\mathcal M) = (T, F, \overline{T \cup F}). \]
\end{itemize}
\end{mdef}

In contrast with $\LabMod$ and $\ModLab$, the functions $\lm$ and $\ml$ are each other's inverse in the general case:

\begin{restatable}{theo}{inverseafsetaf}\label{t:inverse-af-setaf}
Let $\af = (\ar, \att)$ be a $\setaf$ and $P_\af$ its associated $\nlp$.

\begin{itemize}
    \item For any labelling $\nlab$ of $\af$, it holds $\ml(\lm(\nlab)) = \nlab$.
    \item For any interpretation $\inter$ of $P_\af$, it holds $\lm(\ml(\inter)) = \inter$.
\end{itemize}
\end{restatable}

A similar result to Theorem \ref{t:psmcomp} also holds here:

\begin{restatable}{theo}{comppsmsetaf}\label{t:comppsmsetaf}
Let $\af$ be a $\setaf$ and $P_\af$ be its associated $\nlp$. It holds

\begin{itemize}
    \item $\nlab$ is a complete labelling of $\af$ iff $\lm(\nlab)$ is a partial stable model of $P_\af$.
    \item $\mathcal M$ is a partial stable model of $P_\af$ iff $\ml(\mathcal M)$ is a complete labelling of $\af$.
\end{itemize}
\end{restatable}

From Theorem \ref{t:comppsmsetaf}, we can ensure the equivalence between the semantics for $\nlp$ and their counterpart for $\setaf$:

\begin{restatable}{theo}{equivalencesetaf}\label{t:equivalencesetaf}
Let $\af$ be a $\setaf$ and $P_\af$ its associated $\nlp$. It holds

\begin{enumerate}
    \item $\nlab$ is a grounded labelling of $\af$ iff $\lm(\nlab)$ is a well-founded model of $P_\af$.
    \item $\nlab$ is a preferred labelling of $\af$ iff $\lm(\nlab)$ is a regular model of $P_\af$.
    \item $\nlab$ is a stable labelling of $\af$ iff $\lm(\nlab)$ is a stable model of $P_\af$.
    \item $\nlab$ is a semi-stable labelling of $\af$ iff $\lm(\nlab)$ is an $L$-stable model of $P_\af$.
\end{enumerate}

\end{restatable}

The following result is a direct consequence of Theorems \ref{t:inverse-af-setaf} and \ref{t:equivalencesetaf}:

\begin{restatable}{cor}{setafequivalenceII}\label{c:setafequivalenceII}
Let $\af$ be a $\setaf$ and $P_\af$ its associated $\nlp$. It holds

\begin{enumerate}
    \item $\mathcal M$ is a well-founded model of $P_\af$ iff $\ml(\mathcal M)$ is a grounded labelling of $\af$.
    \item $\mathcal M$ is a regular model of $P_\af$ iff $\ml(\mathcal M)$ is a preferred labelling of $\af$.
    \item $\mathcal M$ is a stable model of $P_\af$ iff $\ml(\mathcal M)$ is a stable labelling of $\af$.
    \item $\mathcal M$ is an $L$-stable model of $P_\af$ iff $\ml(\mathcal M)$ is a semi-stable labelling of $\af$.
\end{enumerate}

\end{restatable}

Recalling the $\setaf$ $\af$ and its associated $P_\af$ of Example \ref{ex:setaf-nlp}, we obtain the expected equivalence results related to their semantics (see Table \ref{tab:p-afp}). In the next section, we will identify a class of $\nlp$s in which the translation from a $\setaf$ to an $\nlp$ (Definition \ref{d:setaf-nlp}) behaves as the inverse of the translation from an $\nlp$ to a $\setaf$ (Definition \ref{d:nlp-setaf}).

\section{On the relation between $\rfalp$s and $\setaf$s}\label{s:slp-setaf}

We will recall a particular kind of $\nlp$s, called Redundancy-Free Atomic Logic Programs $(\rfalp s)$. From an $\rfalp$ $P$, we obtain its associated $\setaf$ $\af_P$ via Definition \ref{d:nlp-setaf}; from $\af_P$, we obtain its associated $\nlp$ $P_{\af_P}$ via Definition \ref{d:setaf-nlp}. By following the other direction, from a $\setaf$ $\af$, we obtain its associated $\nlp$ $P_\af$, and from $P_\af$, its associated $\setaf$ $\af_{P_\af}$.  An important result mentioned in this section is that $P = P_{\af_P}$ and $\af = \af_{P_\af}$, i.e., the translation from an $\nlp$ to a $\setaf$ and the translation from a $\setaf$ to an $\nlp$ are each other's inverse. Next, we define $\rfalp$s:

\begin{mdef}[$\rfalp$ \protect\citep{konig22just}]\label{d:slp} We define a Redundancy-Free Atomic Logic Program $(\rfalp)$\@ $P$ as an $\nlp$ such that  
\begin{enumerate}
\item $P$ is redundancy-free, i.e., $\HB_P = \set{\mathit{head}(r) \mid r \in P}$ and if $c \leftarrow \naf b_1, \ldots, \naf b_n \in P$, there is no rule $c \leftarrow \naf c_1, \ldots, \naf c_{n'} \in P$ such that $\set{c_1, \ldots, c_{n'}} \subset \set{b_1, \ldots, b_n}$.
\item $P$ is atomic, i.e., each rule has the form $c \leftarrow \naf b_1, \ldots, \naf b_n$ $(n \geq 0)$.
\end{enumerate}
\end{mdef}

Firstly, Proposition \ref{p:slp} sustains that for any $\setaf$ $\af$, its associated $\nlp$ $P_\af$ will always be an $\rfalp$:

\begin{restatable}{prop}{slp}\label{p:slp} Let $\af = (\ar, \att)$ be a $\setaf$ and $P_\af$ its associated $\nlp$. It holds $P_\af$ is an $\rfalp$.
\end{restatable}

The following results guarantee that $\af = \af_{P_\af}$ (Theorem \ref{t:inverse-setaf-slp}) and $P = P_{\af_P}$ (Theorem \ref{t:inverse-slp-setaf}):

\begin{restatable}{theo}{inversesetafslp}\label{t:inverse-setaf-slp}
Let $\af = (\ar, \att)$ be a $\setaf$, $P_\af$ its associated $\nlp$ and $\af_{P_\af}$ the associated $\setaf$ of $P_\af$. It holds that $\af = \af_{P_\af }$.
\end{restatable}

\begin{restatable}{theo}{inverseslpsetaf}\label{t:inverse-slp-setaf}
Let $P$ be an $\rfalp$, $\af_P$ its associated $\setaf$ and $P_{\af_P}$ the associated $\nlp$ of $\af_P$. It holds that $P = P_{\af_P}$.
\end{restatable}

\begin{rmk}
Minimality is crucial to ensure that the translation from an $\nlp$ to a $\setaf$ and the translation from a $\setaf$ to an $\nlp$ are each other's inverse. If the minimality requirement in Definition \ref{d:setaf} (and consequently in Definition \ref{d:attack}) were dropped, any $\setaf$ (among other combinations) in Fig. \ref{f:min-dropped} could be a possible candidate to be the associated $\setaf$ $\af_P$ of the $\rfalp$ $P$
\[
\begin{array}{lcl}
c \leftarrow \naf a, \naf c	 & & c \leftarrow \naf b, \naf c\\
a  & & b 	
\end{array}
\]

\begin{figure}[ht!]
    \centering
    \begin{subfigure}[b]{0.32\textwidth}
         \centering
\begin{tikzpicture}[>=stealth',shorten >=1pt,auto,node distance=2cm,
  thick,main node/.style={circle,fill=white!20,draw}, inner sep=3pt]

  \node[main node] (c) {$c$};
  \node[main node] (b) [xshift = 1.5cm, below of=c] {$b$};
  \node[main node] (a) [xshift = -1.5cm, below of=c] {$a$};

  \path[every node/.style={font=\sffamily\small}]
    (c) edge [->,loop above]	node {} (c);
    \connectThree[
          @edge 3=->
        ]{a}{b}{c}     
\end{tikzpicture} 
         \caption{$\af_P'$}
         \label{f:min-droppedI}
     \end{subfigure}
     \hfill
    \begin{subfigure}[b]{0.32\textwidth}
         \centering
\begin{tikzpicture}[>=stealth',shorten >=1pt,auto,node distance=2cm,
  thick,main node/.style={circle,fill=white!20,draw}, inner sep=3pt]

  \node[main node] (c) {$c$};
  \node[main node] (b) [xshift = 1.5cm, below of=c] {$b$};
  \node[main node] (a) [xshift = -1.5cm, below of=c] {$a$};

  \path[every node/.style={font=\sffamily\small}]
    (c) edge [->,loop above] node {} (c)
        edge [->,out=150,in=210,looseness=12] node[above] {} (c);  
    
    \connectThree[
          @edge 3=->
        ]{a}{b}{c}   

 \draw (a) to [out=150,in=210]               (-.6,-.29);

\end{tikzpicture} 
         \caption{$\af_P''$}
         \label{f:min-droppedII}
     \end{subfigure}
\begin{subfigure}[b]{0.32\textwidth}
         \centering
\begin{tikzpicture}[>=stealth',shorten >=1pt,auto,node distance=2cm,
  thick,main node/.style={circle,fill=white!20,draw}, inner sep=3pt]
        
  \node[main node] (c) {$c$};
  \node[main node] (b) [xshift = 1.5cm, below of=c] {$b$};
  \node[main node] (a) [xshift = -1.5cm, below of=c] {$a$};
        
  \path[every node/.style={font=\sffamily\small}]
      (c) edge [->,loop above] node {} (c)
        edge [->,out=150,in=210,looseness=12] node[above] {} (c)
        edge [->,out=30,in=-30,looseness=12] node[above] {} (c); 
    \connectThree[
          @edge 3=->
        ]{a}{b}{c}  

 \draw (a) to [out=150,in=210]               (-.6,-.29);       \draw (b) to [out=30,in=-30]               (.6,-.29);     
\end{tikzpicture} 
         \caption{$\af_P'''$}
         \label{f:min-droppedIII}
     \end{subfigure}     
  \caption{Possible $\setaf$s associated with $P$}
        \label{f:min-dropped}   
\end{figure}
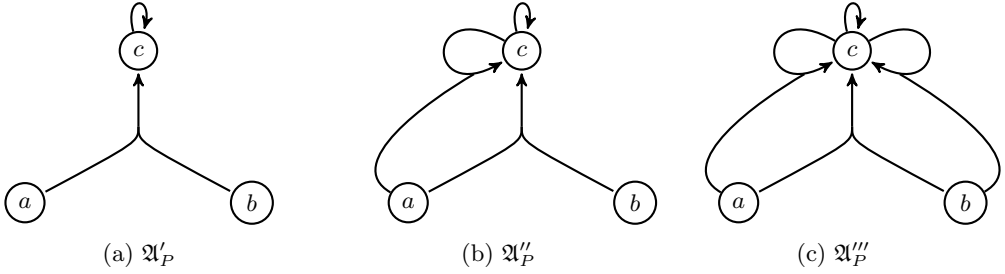

As a result, Theorem \ref{t:inverse-setaf-slp} would no longer hold, and these translations would not be each other's inverse. Notice also that the $\setaf$s in Fig. \ref{f:min-dropped} have the same complete labellings as non-minimal attacks are irrelevant and can be ignored when determining semantics based on complete labellings.    
\end{rmk}

Theorems \ref{t:inverse-setaf-slp} and \ref{t:inverse-slp-setaf} reveal that $\setaf$s and $\rfalp$s are essentially the same formalism. The equivalence between them involves their semantics and is also structural: two distinct $\setaf$s will always be translated into two distinct $\rfalp$s and vice versa. In contradistinction, Theorem \ref{t:inverse-slp-setaf} would not hold if we had replaced our translation from $\nlp$ to $\setaf$ (Definition \ref{d:nlp-setaf}) with that from $\nlp$ to $\aaf$ presented in \citep{caminada15equivalence}. Thus, the connection between $\nlp$s and $\setaf$s is more robust than that between $\nlp$s and $\aaf$s. In the forthcoming section, we will explore how expressive $\rfalp$s can be; we will ensure they are as expressive as $\nlp$s.

\section{On the Expressiveness of $\rfalp$s}\label{s:expressiveness}

 Dvorák et al. comprehensively characterised the expressiveness of $\setaf$s \citep{dvovrak2019expressive}. Now we compare the expressiveness of $\nlp$s with that of $\rfalp$s. In the previous section, we established that $\setaf$s and $\rfalp$s are essentially the same formalism. We demonstrated that from the $\setaf$ $\af_P$ associated with an $\nlp$ $P$, we can obtain $P$; and conversely, from the $\nlp$ $P_\af$ associated with a $\setaf$ $\af$, we can obtain $\af$. Here, we reveal that this connection between $\setaf$s and $\rfalp$s is even more substantial: $\rfalp$s are as expressive as $\nlp$s when considering the semantics for $\nlp$s we have exploited in this paper. With this aim in mind, we transform any $\nlp$ $P$ into an $\rfalp$ $P^*$ by resorting to a specific combination (denoted by $\mapsto_\utpm$) of some program transformations proposed by Brass and Dix \citep{brass1994disjunctive,brass1997characterizations,brass1999semantics}. Although each program transformation in $\mapsto_\utpm$ was proposed in \citep{brass1994disjunctive,brass1997characterizations,brass1999semantics}, the combination of these program transformations (as far as we know) has not been investigated yet. Then, we show that $P$ and $P^*$ share the same partial stable models. Since well-founded models, regular models, stable models, and $L$-stable models are all settled on partial stable models, it follows that both $P$ and $P^*$ also coincide under these semantics. Based on Dunne et al.'s work \citep{dunne2015characteristics}, where they define the notion of expressiveness of the semantics for $\aaf$s, we define formally expressiveness in terms of the signatures of the semantics for $\nlp$s:

\begin{mdef}[Expressiveness]\label{d:expressiveness}
Let $\mathcal P$ be a class of $\nlp$s. The signature $\Sigma^\mathcal P_{\mathit{PSM}}$ of the partial stable models associated with $\mathcal P$ is defined as
\[ \Sigma^\mathcal P_{\mathit{PSM}} = \set{ \sigma(P) \mid P \in \mathcal P}, \]
where $\sigma(P) = \set{\inter \mid \inter \textit{ is a partial stable model of } P}$ is the set of all partial stable models of $P$.

Given two classes $\mathcal P_1$ and $\mathcal P_2$ of $\nlp$s, we say that $\mathcal P_1$ and $\mathcal P_2$ have the same expressiveness for the partial stable models semantics if $\Sigma^{\mathcal P_1}_{\mathit{PSM}} = \Sigma^{\mathcal P_2}_{\mathit{PSM}}$
\end{mdef}

In other words, $\mathcal P_1$ and $\mathcal P_2$ have the same expressiveness if 
\begin{itemize}
\item For every $P_1 \in \mathcal P_1$, there exists $P_2 \in \mathcal P_2$ such that $P_1$ and $P_2$ have the same set of partial stable models. 
\item For every $P_2 \in \mathcal P_2$, there exists $P_1 \in \mathcal P_1$ such that $P_1$ and $P_2$ have the same set of partial stable models.
\end{itemize}

Similarly, we can define when $\mathcal P_1$ and $\mathcal P_2$ have the same expressiveness for the well-founded, regular, stable, and $L$-stable semantics.

As the class of $\rfalp$s is contained in the class of all $\nlp$s, to show that these classes have the same expressiveness for these semantics, it suffices to prove that for every $\nlp$, there exists an $\rfalp$ with the same set of partial stable models. We will obtain this result by resorting to a combination of program transformations:

\begin{mdef}[Program Transformation \protect\citep{brass1994disjunctive,brass1997characterizations,brass1999semantics}] A program transformation is any binary relation $\mapsto$ between $\nlp$s. By $\mapsto^*$ we mean the reflexive and transitive closure of $\mapsto$.
\end{mdef}

Thus, $P \mapsto^* P'$ means that there is a finite sequence $P = P_1 \mapsto \cdots \mapsto P_n = P'$. We are particularly interested in program transformations preserving partial stable models:

\begin{mdef}[Equivalence Transformation \protect\citep{brass1994disjunctive,brass1997characterizations,brass1999semantics}] We say a program transformation $\mapsto$ is a partial stable model equivalence transformation if for any $\nlp$s $P_1$ and $P_2$ with $P_1 \mapsto P_2$, it holds $\mathcal M$ is a partial stable model of $P_1$ iff $\mathcal M$ is a partial stable model of $P_2$.
\end{mdef}

From Definitions \ref{d:unfolding} to \ref{d:eliminination-non-minimal}, we focus on the following program transformations introduced in \citep{brass1994disjunctive,brass1997characterizations,brass1999semantics}: \emph{Unfolding} (\emph{it is also known as Generalised Principle of Partial Evaluation} (\emph{GPPE})), \emph{Elimination of Tautologies}, \emph{Positive Reduction}, and \emph{Elimination of Non-Minimal Rules}. They are sufficient for our purposes.
 
\begin{mdef}[Unfolding \protect\citep{brass1994disjunctive,brass1997characterizations,brass1999semantics}]\label{d:unfolding}
An $\nlp$ $P_2$ results from an $\nlp$ $P_1$ by unfolding ($P_1 \mapsto_U P_2$) iff there exists a rule $c \leftarrow a, a_1, \ldots, a_m, \naf b_1, \ldots, \naf b_n \in P_1$ such that 
\begin{align*}
    P_2 = &\ (P_1 - \set{c \leftarrow a, a_1, \ldots, a_m, \naf b_1, \ldots, \naf b_n})\\ 
    & \hspace{5em} \cup \{c \leftarrow a_1', \ldots, a_p', a_1, \ldots, a_m, \naf b_1', \ldots, \naf b_q', \naf b_1, \ldots, \naf b_n \mid \\
    & \hspace{6.675em} a \leftarrow a_1', \ldots, a_p', \naf b_1', \ldots, \naf b_q' \in P_1 \}.
\end{align*}
\end{mdef}

\begin{mdef}[Elimination of Tautologies \protect\citep{brass1994disjunctive,brass1997characterizations,brass1999semantics}] An $\nlp$ $P_2$ results from an $\nlp$ $P_1$ by elimination of tautologies ($P_1 \mapsto_T P_2$) iff there exists a rule $r \in P_1$ such that $\hd(r) \in \body^+(r)$ and $P_2 = P_1 - \set{r}$. 
\end{mdef}

\begin{mdef}[Positive Reduction \protect\citep{brass1994disjunctive,brass1999semantics}]  An $\nlp$ $P_2$ results from an $\nlp$ $P_1$ by positive reduction ($P_1 \mapsto_P P_2$) iff there exists a rule $c \leftarrow a_1, \ldots, a_m, \naf b, \naf b_1, \ldots, \naf b_n \in P_1$ such that $b \not\in \set{\hd(r) \mid r \in P_1}$ and 
\begin{align*}
    P_2 = &\ (P_1 - \set{c \leftarrow a_1, \ldots, a_m, \naf b, \naf b_1, \ldots, \naf b_n})\\ 
    & \hspace{4.475em} \cup \set{c \leftarrow a_1, \ldots, a_m, \naf b_1, \ldots, \naf b_n}.
\end{align*}
\end{mdef}

\begin{mdef}[Elimination of Non-Minimal Rules \protect\citep{brass1994disjunctive,brass1999semantics}]\label{d:eliminination-non-minimal}   An $\nlp$ $P_2$ results from an $\nlp$ $P_1$ by elimination of non-minimal rules ($P_1 \mapsto_M P_2$) iff there are two distinct rules $r$ and $r'$ in $P_1$ such that $\hd(r) = \hd(r')$, $\body^+(r') \subseteq \body^+(r)$, $\body^-(r') \subseteq \body^-(r)$ and $P_2 = P_1 - \set{r}$.
\end{mdef}

Now we combine these program transformations and define $\mapsto_\utpm$ as follows:

\begin{mdef}[Combined Transformation]\ 
Let $\mapsto_\utpm = \mapsto_U \cup \mapsto_T \cup \mapsto_P \cup \mapsto_M$.
\end{mdef}

We call an $\nlp$ $P$ \emph{irreducible} concerning $\mapsto$ if there is no $\nlp$ $P' \neq P$ with $P \mapsto^* P'$. Besides, we say $\mapsto$ is \emph{strongly terminating} iff every sequence of successive applications of $\mapsto$ eventually leads to an irreducible $\nlp$. As displayed in \citep{brass1998characterizations}, not every program transformation is strongly terminating. For instance, in the $\nlp$
\[
\begin{array}{rcl}
 a & \leftarrow & b\\
b & \leftarrow & a\\
c & \leftarrow & a, \naf c\\
c
\end{array}
\]
if we apply unfolding $(\mapsto_U)$ to the third rule, this rule is replaced by $c \leftarrow b, \naf c$. We can now apply unfolding again to this rule and get the original program; such an oscillation can repeat indefinitely. Thus we have a sequence of program transformations that do not terminate. However, if we restrict ourselves to fair sequences of program transformations, the termination is guaranteed:

\begin{mdef}[Fair Sequences \protect\citep{brass1998characterizations}]\label{d:confluent} 
A sequence of program transformations $P_1 \mapsto \cdots \mapsto P_n$ is fair with respect to $\mapsto$ if 

\begin{itemize}
    \item Every positive body atom occurring in $P_1$ is eventually removed in some $P_i$ with $1 < i \leq n$ (either by removing the whole rule using a suitable program transformation or by an application of $\mapsto_U$);
    \item  Every rule $r \in P_i$ such that $\hd(r) \in \body^+(r)$ is eventually removed in some $P_j$ with $i < j \leq n$ (either by applying $\mapsto_T$ or another suitable program transformation).
\end{itemize}

\end{mdef}

The sequence above of program transformations is not fair, because it does not remove the positive body atoms occurring in the program. In contrast, the sequence of program transformations given by
\[
\small{
\begin{array}{rcl}
a & \leftarrow & b\\
b & \leftarrow & a\\
c & \leftarrow & a, \naf c\\
c
\end{array}
\mapsto_U
\begin{array}{rcl}
a & \leftarrow & a\\
b & \leftarrow & a\\
c & \leftarrow & a, \naf c\\
c
\end{array}
\mapsto_T
\begin{array}{rcl}
b & \leftarrow & a\\
c & \leftarrow & a, \naf c\\
c
\end{array}
\mapsto_U
\begin{array}{rcl}
b & \leftarrow & a\\
c
\end{array}
\mapsto_U
\begin{array}{rcl}
c
\end{array}
}
\]
is not only fair but also terminates. The next result guarantees that it is not simply a coincidence:

\begin{restatable}{theo}{terminates}\label{t:terminates}
The relation $\mapsto_{\utpm}$ is strongly terminating for fair sequences of program transformations, i.e., such fair sequences always lead to irreducible programs.
\end{restatable}

Theorem \ref{t:terminates} is crucial to obtain the following result:

\begin{restatable}{theo}{irreducible}\label{t:irreducible}
For any $\nlp$ $P$, there exists an irreducible $\nlp$ $P^*$ such that $P \mapsto^*_{\utpm} P^*$.
\end{restatable}

This means that from an $\nlp$ $P$, it is always possible to obtain an irreducible $\nlp$ $P^*$ after successive applications of $\mapsto_{\utpm}$. Indeed, $P^*$ is an $\rfalp$:

\begin{restatable}{theo}{utpmsetaf}\label{t:utpmf-setaf}
Let $P$ be an $\nlp$ and $P^*$ be an $\nlp$ obtained after applying repeatedly the program transformation $\mapsto_\utpm$ until no further transformation is possible, i.e., $P \mapsto^*_\utpm P^*$ and $P^*$ is irreducible. Then $P^*$ is an $\rfalp$.
\end{restatable}

From Theorems \ref{t:terminates} and  \ref{t:utpmf-setaf}, we can infer that for fair sequences, after applying repeatedly $\mapsto_\utpm$, we will eventually produce an $\rfalp$. In fact, every $\rfalp$ is irreducible:

\begin{restatable}{theo}{setafutpm}\label{t:setaf-utpmf}
Let $P$ be an $\rfalp$. Then $P$ is irreducible with respect to $\mapsto_\utpm$.
\end{restatable}

Theorems \ref{t:irreducible} and \ref{t:utpmf-setaf} guarantee that every $\nlp$ $P$ can be transformed into an $\rfalp$ $P^*$ by applying $\mapsto_\utpm$ a finite number of times. It remains to show that $P$ and $P^*$ share the same partial stable models (and consequently, the same well-founded, regular, stable, and $L$-stable models). Before, however, note that $\mapsto_\utpm$ does not introduce new atoms; instead, it can eliminate the occurrence of existing atoms in an $\nlp$. For simplicity in notation, we assume throughout the rest of this section that $\HB_P = \HB_{P'}$ whenever $P \mapsto_\utpm^* P'$.   Next, we recall that these program transformations preserve the least models of positive programs:

\begin{lem}[\protect \citep{brass1995disjunctive,brass1997characterizations}]\label{l:definitex}
Let $P_1$ and $P_2$ be positive programs such that $P_1 \mapsto_x P_2$, in which $x \in \set{U,T,P,M}$. It holds $\mathcal M$ is the least model of $P_1$ iff $\mathcal M$ is the least model of $P_2$.
\end{lem}

In the sequel, we aim to extend Lemma \ref{l:definitex} to $\nlp$s. Notice, however, that we already have the result for the program transformation $\mapsto_U$:

\begin{theo}[\protect \citep{aravindan1995correctness}]\label{t:psm-U}
Let $P_1$ and $P_2$ be $\nlp$s such that $P_1 \mapsto_U P_2$. It holds $\mathcal M$ is a partial stable model of $P_1$ iff $\mathcal M$ is a partial stable model of $P_2$.
\end{theo}

It remains to guarantee the result for the program transformation $\mapsto_T$, $\mapsto_P$ and $\mapsto_M$:

\begin{restatable}{theo}{psmT}\label{t:psm-T}
Let $P_1$ and $P_2$ be $\nlp$s such that $P_1 \mapsto_T P_2$. It holds $\mathcal M$ is a partial stable model of $P_1$ iff $\mathcal M$ is a partial stable model of $P_2$.
\end{restatable}

\begin{restatable}{theo}{psmP}\label{t:psm-P}
Let $P_1$ and $P_2$ be $\nlp$s such that $P_1 \mapsto_P P_2$. It holds $\mathcal M$ is a partial stable model of $P_1$ iff $\mathcal M$ is a partial stable model of $P_2$.
\end{restatable}

\begin{restatable}{theo}{psmM}\label{t:psm-M}
Let $P_1$ and $P_2$ be $\nlp$s such that $P_1 \mapsto_M P_2$. It holds $\mathcal M$ is a partial stable model of $P_1$ iff $\mathcal M$ is a partial stable model of $P_2$.
\end{restatable}

Consequently, if $P_1 \mapsto_\utpm P_2$, then $P_1$ and $P_2$ share the same partial stable models. By repeatedly resorting to this result, we can even show that for any $\nlp$, there exists an irreducible $\nlp$ with the same set of partial stable models, well-founded models, regular models, stable models, and $L$-stable models: 

\begin{restatable}{theo}{psmUTPM}\label{t:psm-UTPMF}
Let $P$ be an $\nlp$ and $P^*$ be an irreducible $\nlp$ such that $P \mapsto^*_\utpm P^*$.  It holds $\mathcal M$ is a partial stable model of $P$ iff $\mathcal M$ is a partial stable model of $P^*$.
\end{restatable}

\begin{restatable}{cor}{equivUTPM}\label{c:equiv-UTPMF}
Let $P$ be an $\nlp$ and $P^*$ be an irreducible $\nlp$ such that $P \mapsto^*_\utpm P^*$. It holds $\mathcal M$ is a well-founded, regular, stable, $L$-stable model of $P$ iff $\mathcal M$ is respectively a well-founded, regular, stable, $L$-stable model of $P^*$.
\end{restatable}

As any irreducible $\nlp$ is an $\rfalp$ (Theorem \ref{t:utpmf-setaf}), the following result is immediate:

\begin{restatable}{cor}{lpequivsetaf}\label{c:lpequivsetaf}
For any $\nlp$ $P$, there exists an $\rfalp$ $P^*$ such that $\mathcal M$ is a partial stable, well-founded, regular, stable, $L$-stable model of $P$ iff $\mathcal M$ is respectively a partial stable, well-founded, regular, stable, $L$-stable model of $P^*$.
\end{restatable}

Given that each $\nlp$ can be associated with an $\rfalp$ preserving the semantics above, it follows that $\nlp$ and $\rfalp$s have the same expressiveness for those semantics:

\begin{restatable}{theo}{lpslp}\label{t:lp-slp}
$\nlp$s and $\rfalp$s have the same expressiveness for partial stable, well-founded, regular, stable, and $L$-stable semantics.
\end{restatable}

Another important result is that the $\setaf$ corresponding to an $\nlp$ is invariant with respect to $\mapsto_\utpm$:

\begin{restatable}{theo}{invariant}\label{t:invariant}
For any $\nlp$s $P_1$ and $P_2$, if $P_1 \mapsto_\utpm P_2$, then $\af_{P_1} = \af_{P_2}$ 
\end{restatable}

This means that any $\nlp$ in a sequence of program transformations from $\mapsto_\utpm$ has the same corresponding $\setaf$. For instance, every $\nlp$ in this sequence 
\[
\small{
\begin{array}{rcl}
a & \leftarrow & b\\
b & \leftarrow & a\\
c & \leftarrow & a, \naf c\\
c
\end{array}
\mapsto_U
\begin{array}{rcl}
a & \leftarrow & a\\
b & \leftarrow & a\\
c & \leftarrow & a, \naf c\\
c
\end{array}
\mapsto_T
\begin{array}{rcl}
b & \leftarrow & a\\
c & \leftarrow & a, \naf c\\
c
\end{array}
\mapsto_U
\begin{array}{rcl}
b & \leftarrow & a\\
c
\end{array}
\mapsto_U
\begin{array}{rcl}
c
\end{array}
}
\]
leads to the same corresponding $\setaf$, constituted by a unique (unattacked) argument:
\begin{center}
\begin{tikzpicture}[>=stealth',shorten >=1pt,auto,node distance=2cm,
  thick,main node/.style={circle,fill=white!20,draw}, inner sep=3pt]
  \node[main node] (c) {$c$};
\end{tikzpicture}
\end{center}

Theorem \ref{t:invariant} also suggests an alternative way to find the $\setaf$ corresponding to an $\nlp$ $P$: instead of resorting directly to Definition \ref{d:argument} to construct the arguments, we can apply (starting from $P$) $\mapsto_\utpm$ successively by following a fair sequence of program transformations. By Theorems \ref{t:terminates} and \ref{t:utpmf-setaf}, we know that eventually, we will reach an $\rfalp$ whose corresponding $\setaf$ is identical to that of the original program $P$ (Theorem \ref{t:invariant}). Then, we apply Definition \ref{d:argument} to this $\rfalp$ to obtain the arguments and Definition \ref{d:attack} for the attack relation. Notably, when $P$ is an $\rfalp$, Definition \ref{d:argument} becomes considerably simpler, requiring only its first item to characterise the statements.   

In addition, from the same $\nlp$, various fair sequences of program transformations can be conceived. Recalling the $\nlp$
\[
\begin{array}{rcl}
 a & \leftarrow & b\\
b & \leftarrow & a\\
c & \leftarrow & a, \naf c\\
c
\end{array}
\]
exploited above, we can design the following alternative fair sequence
\[
\small{
\begin{array}{rcl}
a & \leftarrow & b\\
b & \leftarrow & a\\
c & \leftarrow & a, \naf c\\
c
\end{array}
\mapsto_M
\begin{array}{rcl}
a & \leftarrow & a\\
b & \leftarrow & a\\
c
\end{array}
\mapsto_U
\begin{array}{rcl}
a & \leftarrow & b\\
b & \leftarrow & b\\
c
\end{array}
\mapsto_T
\begin{array}{rcl}
a & \leftarrow & b\\
c
\end{array}
\mapsto_U
\begin{array}{rcl}
c
\end{array}
}
\]
This sequence produced the same $\rfalp$ as before; it is not a coincidence. Apart from being strongly terminating for fair sequences of program transformations, the relation $\mapsto_\utpm$ has an appealing property; it is also confluent:

\begin{restatable}{theo}{confluent}\label{t:confluent}
The relation $\mapsto_\utpm$ is confluent, i.e., for any $\nlp$s $P$, $P'$ and $P''$, if $P \mapsto_\utpm^* P'$ and $P \mapsto_\utpm^* P''$ and both $P'$ and $P''$ are irreducible, then $P' = P''$.   
\end{restatable}

By confluent $\mapsto_\utpm$, we mean that it does not matter the path we take by repeatedly applying $\mapsto_\utpm$, if it ends, it will always lead to the same irreducible $\nlp$. In addition, as any irreducible $\nlp$ is an $\rfalp$ (Theorem \ref{t:utpmf-setaf}), and the translations from $\setaf$ to $\rfalp$s and conversely, from $\rfalp$s to $\setaf$ are each other's inverse (Theorems \ref{t:inverse-setaf-slp} and \ref{t:inverse-slp-setaf}), we obtain that two distinct $\setaf$s will always be associated with two distinct $\nlp$s. The confluence of $\mapsto_\utpm$ is of particular significance from the logic programming perspective as it guarantees that the ordering of the transformations in $\utpm$ does not matter: we are free to choose always the “best” transformation, which maximally reduces the program. Consequently, Theorem \ref{t:confluent} also sheds light on the search for efficient implementations in $\nlp$s.

From the previous section, we know that the equivalence between $\setaf$s and $\rfalp$s is not only of a semantic nature but also structural: two distinct $\setaf$s will always be translated into two distinct $\rfalp$s and vice versa. Now we enhance our understanding of this result still more by establishing that 

\begin{itemize}
\item $\rfalp$s are as expressive as $\nlp$s. 
\item The $\setaf$ corresponding to an $\nlp$ is invariant with respect to $\mapsto_\utpm$, i.e., if $P_1 \mapsto_\utpm P_2$, then $\af_{P_1} = \af_{P_2}$.
\item Each $\nlp$ $P$ leads to a unique $\rfalp$ $P^*$ via the relation $\mapsto_\utpm$. Besides, $P$ and $P^*$ have the same partial stable, grounded, regular, stable, and $L$-stable models.
\end{itemize}

Beyond revealing the connections between $\setaf$s and $\nlp$s, the results in this paper also enhance our understanding of $\nlp$s themselves. To give a concrete example, let us consider the following issue: in the sequence of program transformations in $\mapsto_\utpm$, atoms can be removed. Are these atoms underivable and set to false in the partial stable models of the program or true/undecided atoms can be removed in this sequence? Such questions can be answered by considering some results from Section \ref{s:nlp-setaf} and the current section. In more formal terms, let $P$ and $P^*$ be $\nlp$s such that $P \mapsto^*_\utpm P^*$ and $P^*$ is irreducible. We have

\begin{itemize}
\item $P^*$ is an $\rfalp$ (Theorem \ref{t:utpmf-setaf}), and the set of atoms occurring in $P^*$ is $\set{\mathit{head}(r) \mid r \in P^*}$ (Definition \ref{d:slp});
\item $\ar_{P^*} = \set{\mathit{head}(r) \mid r \in P^*}$ is the set of all arguments we can construct from $P^*$ (Definition \ref{d:argument}), and $\af_P = \af_{P^*}$ (Theorem \ref{t:invariant}), i.e., $\ar_P = \ar_{P^*}$;
\item Thus $c$ occurs in $P$, but does not occur in $P^*$ iff there is no statement $s$ constructed from $P$ such that $\Conc(s) = c$. According to Corollary \ref{c:statement-reduct}, $c \in F'$ for every interpretation $\inter$ with $\Omega_P(\inter) = \pair{T', F'}$. \end{itemize}
Consequently, every atom occurring in $P$, but not occurring in $P^*$ is set to false in the least 3-valued model of each disjunct of $P$. In particular, they will be false in its partial stable models.

Supported by the findings presented in the current section, we can argue that $\setaf$s and $\rfalp$s are essentially the same paradigm, and both are deeply connected with $\nlp$s.  

\section{Conclusion and Future Works}\label{s:conclusion}

This paper investigates the connections between frameworks with sets of attacking arguments ($\setaf$s) and Normal Logic Programs ($\nlp$s). Building on the research in \citep{alcantara2019equivalence,alcantara2021equivalence}, we employ the characterisation of the $\setaf$ semantics in terms of labellings \citep{flouris2019comprehensive} to establish a mapping from $\nlp$s to $\setaf$s (and vice versa). We further demonstrate the equivalence between partial stable, well-founded, regular, stable, and $L$-stable models semantics for $\nlp$s and respectively complete, grounded, preferred, stable, and semi-stable labellings for $\setaf$s. 

Our translation from $\nlp$s to
$\setaf$s offers a key advantage over the translation from $\nlp$s to $\aaf$s presented in \citep{caminada15equivalence}. Our approach captures the equivalence between semi-stable labellings for $\setaf$s and $L$-stable models for $\nlp$s. In addition, their translation is unable to preserve the structure of the $\nlp$s. While an $\nlp$ can be translated to an $\aaf$, recovering the original $\nlp$ from the corresponding $\aaf$ is generally not possible. In contradistinction, we have revisited a class of $\nlp$s called Redundancy-Free Atomic Logic Programs ($\rfalp$s). For $\rfalp$s, the translations from $\nlp$s to $\setaf$s, and from $\setaf$s to $\nlp$s also preserve their structures as they are each other's inverse. Hence, when compared to the relationship between $\nlp$s and $\aaf$s, the relationship between $\nlp$s and $\setaf$s is demonstrably more robust. It extends beyond semantics to encompass structural aspects.

Some of these results are not new as they have already been obtained independently in \citep{konig22just}. In fact, their translation from $\nlp$s to $\setaf$s and vice versa coincide with ours, and the structural equivalence between $\rfalp$s and $\setaf$s has also been identified there. Notwithstanding, our proofs of these results stem from a significantly distinct path as they are based on properties of argument labellings and are deeply rooted in works such as \citep{caminada15equivalence,alcantara2019equivalence,alcantara2021equivalence}. For instance, our equivalence results are settled on two important aspects:

\begin{itemize}
\item Properties involving the maximisation/minimisation of labellings adapted from \citep{caminada15equivalence} to deal with labellings for $\setaf$s.
\item Again inspired in \citep{caminada15equivalence}, we proposed a mapping from interpretations to labellings and a mapping from labellings to interpretations. We also showed that they are each other's inverse.
\end{itemize}

In contrast, in \citep{konig22just} the equivalence between the semantics is demonstrated in terms of extensions. They also have not tackled the controversy between semi-stable and $L$-stable, one of our leading motivations for developing this work. 

In addition to showing this structural equivalence between $\rfalp$s and $\setaf$s, we have also investigated the expressiveness of $\rfalp$s. To demonstrate that they are as expressive as $\nlp$s, we proved that any $\nlp$ can be transformed into an $\rfalp$ with the same partial stable models through repeated applications of the program transformation $\mapsto_\utpm$. It is worth noticing that $\mapsto_\utpm$ results from the combination of the following program transformations presented in  \citep{brass1994disjunctive,brass1997characterizations,brass1999semantics}: unfolding, elimination of tautologies, positive reduction, and elimination of non-minimal rules. In the course of our investigations, we also have obtained relevant findings as follows:

\begin{itemize}
\item $\rfalp$s are irreducible with respect to $\mapsto_\utpm$: the application of $\mapsto_\utpm$ to an $\rfalp$ will result in the same program.
\item The mapping from $\nlp$s to $\setaf$s is invariant with respect to the program transformation $\mapsto_\utpm$, i.e., if an $\nlp$ $P_2$ is obtained from an $\nlp$ $P_1$ via $\mapsto_\utpm$, then the $\setaf$ corresponding to $P_1$ is the same corresponding to $P_2$.
\item The program transformation $\mapsto_\utpm$ is confluent: any $\nlp$ will lead to a unique $\rfalp$ after repeatedly applying $\mapsto_\utpm$. Consequently, two distinct $\rfalp$s will always be associated with two distinct $\nlp$s.
\end{itemize}

In summary, $\rfalp$s (which are as expressive as $\nlp$s) and $\setaf$s are essentially the same formalism. Roughly speaking, we can consider a $\setaf$ as a graphical representation of an $\rfalp$, and an $\rfalp$ as a rule-based representation of a $\setaf$. Any change in one formalism is mirrored by a corresponding change in the other. Thus, $\setaf$s emerge as a natural candidate for representing argumentation frameworks corresponding to $\nlp$s.

Regarding the significance and potential impact of our results, we highlight that by pursuing this line of research, one gains insight into what forms of non-monotonic reasoning can and cannot be represented by formal argumentation. In particular, by enlightening these connections between $\setaf$s and $\nlp$s, many approaches, semantics and techniques naturally developed for the former may be applied to the latter, and vice versa. While $\setaf$s serve as an inspiration for defining $\rfalp$s, the representation of $\nlp$s as $\setaf$s is an alternative for intuitively visualising logic programs. 
  
 In addition, our results associated with the confluence of $\mapsto_\utpm$ are of particular significance from the logic programming perspective as they guarantee that the ordering of the transformations in $\mapsto_\utpm$ does not matter: we are free to choose always the “best”
transformation, which maximally reduces the program. Consequently, our paper also sheds light on the search for efficient implementations in $\nlp$s.

Natural ramifications of this work include an in-depth analysis of other program transformations beyond those studied here and their impact on $\setaf$ and argumentation in general. Given the close relationship between Argumentation and Logic Programming, a possible line of research is to investigate how Argumentation can benefit from these program transformations in the development of more efficient algorithms. The structural connection involving $\rfalp$s and $\setaf$s gives rise to exploiting other extensions of Dung $\aaf$s; in particular, we are interested in identifying which of them are robust enough to preserve the structure of logic programs. Along this same line of research, it is also our aim to study connections between extensions of $\nlp$s (including their paraconsistent semantics) and Argumentation.
\medskip  

\textit{Competing interests}: The authors declare none.

\appendix

\section{Proofs of Theorems}\label{s:proofs}

\subsection{Theorems and Proofs from Section \ref{s:nlp-setaf}}\label{ss:nlp-setaf-proofs}

\statementreduct* 

\begin{proof}

\begin{itemize}
\item Proving that $c \in T'$ iff there exists a statement $s$ constructed from $P$ such that $\Conc(s) = c$ and $\Vul(s) \subseteq F$:

\begin{description}
\item[$\Rightarrow$]

Consider $\Psi^{\uparrow\ i}_\frac{P}{\inter} = \pair{T_i, F_i}$ for each $i \in \mathbb{N}$. It suffices to prove by induction on the value of $i$ that if $c \in T_i$, then there exists a statement $s$ constructed from $P$ such that $\Conc(s) = c$ and $\Vul(s) \subseteq F$:

    \begin{itemize}
        \item \emph{Basis}. For $i = 0$, the result is trivial as $T_0 = \emptyset$.
        
        \item \emph{Step}. Assume that for every $c' \in T_n$, there exists a statement $s'$ constructed from $P$ such that $\Conc(s') = c'$ and $\Vul(s') \subseteq F$. We will prove that if $c \in T_{n+1}$, there exists a statement $s$ constructed from $P$ such that $\Conc(s) = c$ and $\Vul(s) \subseteq F$: 

        If $c \in T_{n + 1}$, there exists a rule $c \leftarrow a_1, \ldots, a_m, \naf b_1, \ldots, \naf b_n (m \geq 0, n \geq 0) \in P$ such that $\set{a_1, \ldots, a_m} \subseteq T_n$ and $\set{b_1, \ldots, b_n} \subseteq F$. It follows via inductive step that for every $j \in \set{1, \ldots, m}$, there exists a statement $s_j$ constructed from $P$ such that $\Conc(s_j) = a_j$ and $\Vul(s_j) \subseteq F$. But then, we can construct from $P$ a statement $s$ with $\Conc(s) = c$ where $\Vul(s) = \Vul(s_1) \cup \cdots \cup \Vul(s_m) \cup \set{b_1, \ldots, b_n}$. This implies that $\Vul(s) \subseteq F$.
    \end{itemize}

\item[$\Leftarrow$] We will prove by structural induction on the construction of statements that for each statement $s$ constructed from $P$ such that $\Vul(s) \subseteq F$, it holds $\Conc(s) \in T'$:
    
    \begin{itemize}
        \item \emph{Basis}. Let $s$ be a statement $c \leftarrow \naf b_1, \ldots, \naf b_n$ ($n \geq 0$) such that $\set{b_1, \ldots, b_n} = \Vul(s) \subseteq F$. It follows the fact $c \in \frac{P}{\inter}$. Then $c \in T'$.
        
        \item \emph{Step}. Assume $s_1, \ldots, s_m$ ($m \geq 1$) are arbitrary  statements constructed from $P$ such that for each $i \in \set{1, \ldots, m}$, if $\Vul(s_i) \subseteq F$, then $\Conc(s_i) \in T'$. We will prove that if $s$ is a statement $c \leftarrow (s_1),\ldots,(s_m), \naf b_1,\ldots,
	\naf b_n$ $(n \geq 0)$ constructed from $P$ such that $\Vul(s) \subseteq F$, then $c \in T'$:
 
Let $s$ be such a statement. By Definition \ref{d:argument}, there exists a rule $c \leftarrow a_1, \ldots, a_m, \naf b_1, \ldots, \naf b_n \in P$ such that $\Conc(s_i) = a_i$ for each $i \in \set{1, \ldots, m}$ and $\Vul(s) = \Vul(s_1) \cup \cdots \cup \Vul(s_m) \cup \set{b_1, \ldots, b_n}$. As $\Vul(s) \subseteq F$, we obtain $\set{b_1, \ldots, b_n} \subseteq F$ and $\Vul(s_i) \subseteq F$ for each $i \in \set{1, \ldots, m}$. By inductive hypothesis, it follows $\set{a_1, \ldots, a_m} \subseteq T'$. Then $c \in T'$.
    \end{itemize}

\end{description}

\item Proving that $c \in F'$ iff for every statement $s$ constructed from $P$ such that $\Conc(s) = c$, we have $\Vul(s) \cap T \neq \emptyset$: 

\begin{description}

\item[$\Rightarrow$] Firstly, we will prove by structural induction on the construction of statements that for each statement $s$ constructed from $P$ such that $\Vul(s) \cap T = \emptyset$, it holds $\Conc(s) \not\in F'$:
         
    \begin{itemize}
        \item \emph{Basis}. Let $s$ be a statement $c \leftarrow \naf b_1, \ldots, \naf b_n$ ($n \geq 0$) such that $\set{b_1, \ldots, b_n} \cap T = \emptyset$. It follows the fact $c \in \frac{P}{\inter}$ or $c \leftarrow \mathbf{u} \in \frac{P}{\inter}$. Then $c \not\in F'$.
        
        \item \emph{Step}. Assume $s_1, \ldots, s_m$ ($m \geq 1$) are arbitrary  statements constructed from $P$ such that for each $i \in \set{1, \ldots, m}$, if $\Vul(s_i) \cap T = \emptyset$, then $\Conc(s_i) \not\in F'$. We will prove that if $s$ is a statement $c \leftarrow (s_1),\ldots,(s_m), \naf b_1,\ldots,
	\naf b_n$ $(n \geq 0)$ constructed from $P$ such that $\Vul(s) \cap T = \emptyset$, then $c \not\in F'$:

 Let $s$ be such a statement. By Definition \ref{d:argument}, there exists a rule $c \leftarrow a_1, \ldots, a_m, \naf b_1, \ldots, \naf b_n \in P$ such that $\Conc(s_i) = a_i$ for each $i \in \set{1, \ldots, m}$ and $\Vul(s) = \Vul(s_1) \cup \cdots \cup \Vul(s_m) \cup \set{b_1, \ldots, b_n}$. As $\Vul(s) \cap T = \emptyset$, we obtain $\set{b_1, \ldots, b_n} \cap T = \emptyset$ and $\Vul(s_i) \cap T = \emptyset$ for each $i \in \set{1, \ldots, m}$. By inductive hypothesis, it follows $\set{a_1, \ldots, a_m} \cap F' = \emptyset$. Then, $c \not\in F'$.
    \end{itemize}

Hence, if $c \in F'$, for every statement $s$ constructed from $P$ such that $\Conc(s) = c$, we have $\Vul(s) \cap T \neq \emptyset$.

\item[$\Leftarrow$] Assume that for every statement $s$ constructed from $P$ such that $\Conc(s) = c$, we have $\Vul(s) \cap T \neq \emptyset$. The proof is by contradiction: suppose that $c \not\in F'$. 

Consider $\Psi^{\uparrow\ i}_\frac{P}{\inter} = \pair{T_i, F_i}$ for each $i \in \mathbb{N}$. It suffices to prove by induction on the value of $i$ that if $c \not\in F_i$, then there exists a statement $s$ constructed from $P$ such that $\Conc(s) = c$ and $\Vul(s) \cap T = \emptyset$:

    \begin{itemize}
        \item \emph{Basis}. For $i = 0$, the result is trivial as $F_0 = \HB_P$.
        
        \item \emph{Step}. Assume that for every $c' \not\in F_n$, there exists a statement $s'$ constructed from $P$ such that $\Conc(s') = c'$ and $\Vul(s') \cap T = \emptyset$. We will prove that if $c \not\in F_{n+1}$, there exists a statement $s$ constructed from $P$ such that $\Conc(s) = c$ and $\Vul(s) \cap T = \emptyset$: 

        If $c \not\in F_{n + 1}$, there exists a rule $c \leftarrow a_1, \ldots, a_m, \naf b_1, \ldots, \naf b_n (m \geq 0, n \geq 0) \in P$ such that $\set{a_1, \ldots, a_m} \cap F_n = \emptyset$ and $\set{b_1, \ldots, b_n} \cap T = \emptyset$. It follows via inductive step that for every $j \in \set{1, \ldots, m}$, there exists a statement $s_j$ constructed from $P$ such that $\Conc(s_j) = a_j$ and $\Vul(s_j) \cap T = \emptyset$. But then, we can construct from $P$ a statement $s$ with $\Conc(s) = c$ where $\Vul(s) = \Vul(s_1) \cup \cdots \cup \Vul(s_m) \cup \set{b_1, \ldots, b_n}$. This implies that $\Vul(s) \cap T = \emptyset$.
    \end{itemize}

\end{description}
\end{itemize}
\hfill 
\end{proof}

\inverse*

\begin{proof}
    Let $c \in \ar_P$ and $\LabMod(\mathcal L) = \left<T, F \right>$; there are three possibilities:
    
    \begin{itemize}
        \item $\mathcal L(c) = \inn$ $\Rightarrow$ $c \in T$ $\Rightarrow$ $\ModLab(\LabMod(\mathcal L))(c) = \inn$.
        
        \item $\mathcal L(c) = \out$ $\Rightarrow$ $c \in F$ $\Rightarrow$ $\ModLab(\LabMod(\mathcal L))(c) = \out$.
        
        \item $\mathcal L(c) = \undec$ $\Rightarrow$ $c \in \overline{T \cup F}$ $\Rightarrow$ $\ModLab(\LabMod(\mathcal L))(c) = \undec$.
    \end{itemize}
\hfill 
\end{proof}

\ModArg*

\begin{proof}

Let $\mathcal M = \pair{T, F}$ be a partial stable model of $P$, $\LabMod(\ModLab(\mathcal M)) = \pair{T', F'}$ and $c \in \HB_P$. It suffices to prove the following results:

\begin{itemize}
\item $c \in T$ iff $c \in T'$.

\begin{itemize}

\item Assume $c \in T$. As $\Omega_P(\mathcal M) = \mathcal M$, by Lemma \ref{l:statement-reduct}, there exists a statement $s$ with $\Conc(s) = c$ such that $\Vul(s) \subseteq F$. In particular, it follows that $c \in \ar_P$. This implies  $\ModLab(\mathcal M)(c) = \inn$ and $c \in T'$. 
    \item Assume $c \in T'$. Then $c \in \ar_P$ and $\ModLab(\mathcal M)(c) = \inn$. From Definition \ref{d:labmod}, we obtain $c \in T$.
\end{itemize}

\item $c \in F$ iff $c \in F'$.

\begin{itemize}

\item Assume $c \not\in F'$. Then $c \in \ar_P$ and $\ModLab(\mathcal M)(c) \neq \out$. From Definition \ref{d:labmod}, we obtain $c \not\in F$.
\item Assume $c \not\in F$. As $\Omega_P(\mathcal M) = \mathcal M$, by Lemma \ref{l:statement-reduct}, there exists a statement $s$ with $\Conc(s) = c$ such that $\Vul(s) \cap T = \emptyset$. In particular, it follows that $c \in \ar_P$. This implies  $\ModLab(\mathcal M)(c) \neq \out$ and $c \not\in F'$.

\end{itemize}
\end{itemize}
\hfill 
\end{proof}

\begin{lem}\label{l:vul} Let $P$ be an $\nlp$, $\af_P = (\ar_P, \att_P)$ be the associated $\setaf$ and $v \in \set{\inn, \out, \undec}$. It holds that 

\begin{itemize}
\item For each $\mathcal B \in \att(c)$, $\nlab(b) = v$ for some $b \in \mathcal B$ iff there exists $V \in \Vul(c)$ such that $\nlab(b) = v$ for every $b \in \ar_P \cap V$.
\item For each $\mathcal B \in \att(c)$, $\nlab(b) \neq v$ for some $b \in \mathcal B$ iff there exists $V \in \Vul(c)$ such that $\nlab(b) \neq v$ for every $b \in \ar_P \cap V$.
\end{itemize}

\end{lem}

\begin{proof}
We will prove the result in the first item; the proof of the other result follows a similar path:

\begin{description}
\item[$\Rightarrow$] Assume that for each $\mathcal B \in \att(c)$, $\nlab(b) = v$ for some $b \in \mathcal B$.

By absurd, suppose that for each $V \in \Vul(c)$, it holds that $\nlab(b) \neq v$ for some $b \in \ar_P \cap V$. Then we can construct a set $\mathcal B' \subseteq \ar_P$ by selecting for each $V \in \Vul(c)$, an element $b \in \mathcal V$ such that $\nlab(b) \neq v$. From Definition \ref{d:attack}, we know that there exists $\mathcal B \subseteq \mathcal B'$ such that $(\mathcal B, c) \in \att_P$. But then, there exists $\mathcal B \in \att(c)$ such that $\nlab(b) \neq v$ for each $b \in \mathcal B$. It is absurd as it contradicts our hypothesis. 

 \item[$\Leftarrow$] Assume that there exists $V \in \Vul(c)$ such that $\nlab(b) = v$ for every $b \in \ar_P \cap V$.

 The result is immediate as according to Definition \ref{d:attack}, every set $\mathcal B$ of arguments attacking $c$ contains an element $b \in \ar_P \cap V$.
\end{description}
\hfill 
\end{proof}

\psmcomp*

\begin{proof}
\begin{enumerate}
\item \label{i:psmcompI} If $\nlab$ is a complete labelling of $\af_P$, then $\LabMod(\nlab)$ is a partial stable model of $P$: \medskip

Let $\mathcal M = \LabMod(\nlab) = \pair{T, F}$. We will show $\mathcal M$ is a partial stable model of $P$, i.e., $\Omega_P(\mathcal M) = \pair{T', F'} = \pair{T, F}$:

\begin{itemize}
\item  $c \in T$ iff $c \in \ar_P$ and $\nlab(c) = \inn$ iff for each $\mathcal B \in \att(c)$, it holds $\nlab(b) = \out$ for some $b \in \mathcal B$ iff (Lemma \ref{l:vul}) there exists $V \in \Vul(c)$ such that $\nlab(b) = \out$ for every $b \in \ar_P \cap V$ iff there exists a statement $s$ with $\Conc(s) = c$ and $\Vul(s) \subseteq F$ iff (Lemma \ref{l:statement-reduct}) $c \in T'$.
\item $c \not\in F$ iff $c \in \ar_P$ and $\nlab(c) \neq \out$ iff for each $\mathcal B \in \att(c)$, it holds $\nlab(b) \neq \inn$ for some $b \in \mathcal B$ iff (Lemma \ref{l:vul}) there exists $V \in \Vul(c)$ such that $\nlab(b) \neq \inn$ for every $b \in \ar_P \cap V$ iff there exists a statement $s$ with $\Conc(s) = c$ and $\Vul(s) \cap T = \emptyset$ iff (Lemma \ref{l:statement-reduct}) $c \not\in F'$.
\end{itemize}
\item \label{i:psmcompII} If $\mathcal M$ is a partial stable model of $P$, then $\ModLab(\mathcal M)$ is a complete labelling of $\af_P$: \medskip

Let $\mathcal M = \pair{T, F}$ be a partial stable model of $P$. Then $\Omega_P(\mathcal M) = \pair{T, F}$. Let $c$ be an argument in $\ar_P$. We will prove $\mathcal L = \ModLab(\mathcal M)$ is a complete labelling of $\af_P$:

\begin{itemize}
    \item $\mathcal L(c) = \inn$ iff $c \in T$ iff (Lemma \ref{l:statement-reduct}) there exists a statement $s$ with $\Conc(s) = c$ and $\Vul(s) \subseteq F$ iff there exists $V \in \Vul(c)$ such that $\nlab(b) = \out$ for every $b \in \ar_P \cap V$ iff (Lemma \ref{l:vul}) for each $\mathcal B \in \att(c)$, it holds $\nlab(b) = \out$ for some $b \in \mathcal B$. 
 
    \item $\mathcal L(c) \neq \out$ iff $c \neq F$ iff (Lemma \ref{l:statement-reduct}) there exists a statement $s$ with $\Conc(s) = c$ and $\Vul(s) \cap T = \emptyset$ iff there exists $V \in \Vul(c)$ such that $\nlab(b) \neq \inn$ for every $b \in \ar_P \cap V$ iff (Lemma \ref{l:vul}) for each $\mathcal B \in \att(c)$, it holds $\nlab(b) \neq \inn$ for some $b \in \mathcal B$. 
\end{itemize}
\item If $\LabMod(\nlab)$ is a partial stable model of $P$, then $\nlab$ is a complete labelling of $\af_P$: \medskip

It holds that $\LabMod(\nlab)$ is a partial stable model of $P$ $\Rightarrow$ according to item \ref{i:psmcompII} above, $\ModLab(\LabMod(\nlab))$ is a complete labelling of $\af_P$ $\Rightarrow$ (via Theorem \ref{t:inverse}) $\nlab$ is a complete labelling of $\af_P$.

\item If $\ModLab(\mathcal M)$ is a complete labelling of $\af_P$, then $\mathcal M$ is a partial stable model of $P$: \medskip

It holds that $\ModLab(\mathcal M)$ is a complete labelling of $\af_P$ $\Rightarrow$ according to item \ref{i:psmcompI} above, $\LabMod(\ModLab(\mathcal M))$ is a partial stable model of $P$ $\Rightarrow$ (via Theorem \ref{t:modarg}) $\mathcal M$ is a partial stable model of $P$.
\end{enumerate}
\hfill 
\end{proof}

\begin{lem}\label{l:in-lab-mod}
Let $P$ be an $\nlp$, $\af_P = (\ar_P, \att_P)$ be its associated $\setaf$. Let $\nlab_1$ and $\nlab_2$ be $\beta$-complete labellings of $\af_P$, and $\LabMod(\nlab_1) = \pair{T_1, F_1}$ and $\LabMod(\nlab_2) = \pair{T_2, F_2}$. It holds

\begin{enumerate}
    \item $\inn(\nlab_1) \subseteq \inn(\nlab_2)$ iff $T_1 \subseteq T_2$;
    \item $\inn(\nlab_1) = \inn(\nlab_2)$ iff $T_1 = T_2$;
    \item $\inn(\nlab_1) \subset \inn(\nlab_2)$ iff $T_1 \subset T_2$.
\end{enumerate}
\end{lem}

\begin{proof}\ 
\begin{enumerate}
    \item
\begin{enumerate}
    \item[$(\Rightarrow)$: ] Suppose $\inn(\nlab_1) \subseteq \inn(\nlab_2)$. If $c \in T_1$, by Definition \ref{d:labmod}, $c \in \ar_P$ and $\nlab_1(A) = \inn$. From our initial assumption, it follows $\nlab_2(c) = \inn$. So, by Definition \ref{d:labmod}, $c \in T_2$.
    \item[$(\Leftarrow)$: ] Suppose $T_1 \subseteq T_2$. If $\nlab_1(c) = \inn$, by Definition \ref{d:labmod}, $c \in T_1$. From our initial assumption, it follows $c \in T_2$. So, by Definition \ref{d:labmod}, $\nlab_2(c) = \inn$.
\end{enumerate}    
    \item It follows directly from point 1.
    \item It follows directly from points 1 and 2.
\end{enumerate}
\hfill 
\end{proof}

\begin{lem}\label{l:out-lab-mod}
Let $P$ be an $\nlp$, $\af_P = (\ar_P, \att_P)$ be its associated $\setaf$. Let $\nlab_1$ and $\nlab_2$ be complete labellings of $\af_P$, and $\LabMod(\nlab_1) = \pair{T_1, F_1}$ and $\LabMod(\nlab_2) = \pair{T_2, F_2}$. It holds

\begin{enumerate}
    \item $\out(\nlab_1) \subseteq \out(\nlab_2)$ iff $F_1 \subseteq F_2$;
    \item $\out(\nlab_1) = \out(\nlab_2)$ iff $F_1 = F_2$;
    \item $\out(\nlab_1) \subset \out(\nlab_2)$ iff $F_1 \subset F_2$.
\end{enumerate}
\end{lem}

\begin{proof}\ 
\begin{enumerate}
    \item
\begin{enumerate}
    \item[$(\Rightarrow)$: ] Suppose $\out(\nlab_1) \subseteq \out(\nlab_2)$. If $c \in F_1$, by Definition \ref{d:labmod}, there are two possibilities:
    \begin{itemize}
    \item $c \not\in \ar_P$. As $\LabMod(\nlab_2) = \pair{T_2, F_2}$, we obtain that $c \in F_2$.
    \item $c \in \ar_P$ and $\nlab_1(c) = \out$. From our initial assumption, it follows $\nlab_2(c) = \out$. So, by Definition \ref{d:labmod}, $c \in F_2$.
 
    \end{itemize}
    
    \item[$(\Leftarrow)$: ] Suppose $F_1 \subseteq F_2$. If $\nlab_1(c) = \out$, by Definition \ref{d:labmod}, $c \in F_1$. From our initial assumption, it follows $c \in F_2$. So, by Definition \ref{d:labmod}, $\nlab_2(c) = \out$.
\end{enumerate} 
    \item It follows directly from point 1.
    \item It follows directly from points 1 and 2.
\end{enumerate}
\hfill 
\end{proof}

\begin{lem}\label{l:undec-lab-mod}
Let $P$ be an $\nlp$, $\af_P = (\ar_P, \att_P)$ be its associated $\setaf$. Let $\nlab_1$ and $\nlab_2$ be complete labellings of $\af_P$, and $\LabMod(\nlab_1) = \pair{T_1, F_1}$ and $\LabMod(\nlab_2) = \pair{T_2, F_2}$. It holds

\begin{enumerate}
    \item $\undec(\nlab_1) \subseteq \undec(\nlab_2)$ iff $\overline{T_1 \cup F_1} \subseteq \overline{T_2 \cup F_2}$;
    \item $\undec(\nlab_1) = \undec(\nlab_2)$ iff $\overline{T_1 \cup F_1} = \overline{T_2 \cup F_2}$;
    \item $\undec(\nlab_1) \subset \undec(\nlab_2)$ iff $\overline{T_1 \cup F_1} \subset \overline{T_2 \cup F_2}$.
\end{enumerate}
\end{lem}

\begin{proof}\
\begin{enumerate}
    \item
\begin{enumerate}
    \item[$(\Rightarrow)$: ] Suppose $\undec(\nlab_1) \subseteq \undec(\nlab_2)$. If $c \in \overline{T_1 \cup F_1}$, by Definition \ref{d:labmod}, $c \in \ar_P$ and $\nlab_1(c) = \undec$. From our initial assumption, it follows $\nlab_2(c) = \undec$. So, by Definition \ref{d:labmod}, $c \in \overline{T_2 \cup F_2}$.
    \item[$(\Leftarrow)$: ] Suppose $\overline{T_1 \cup F_1} \subseteq \overline{T_2 \cup F_2}$. If $\nlab_1(c) = \undec$, by Definition \ref{d:labmod}, $c \in \overline{T_1 \cup F_1}$. From our initial assumption, it follows $c \in \overline{T_2 \cup F_2}$. So, by Definition \ref{d:labmod}, $\nlab_2(c) = \undec$.
\end{enumerate} 
    \item It follows directly from point 1.
    \item It follows directly from points 1 and 2. 
\end{enumerate}
\hfill 
\end{proof}

\equivalence*

\begin{proof}
Let $\nlab$ be an argument labelling of $\af_P$ and $\LabMod(\nlab) = \pair{T,F}$. The proof is straightforward:

\begin{enumerate}
    \item $\nlab$ is a grounded labelling of $\af_P$ iff $\nlab$ is a complete labelling of $\af_P$, and $\inn(\nlab)$ is minimal  (w.r.t. set inclusion) among all complete labellings of $\af_P$ iff (Theorem \ref{t:psmcomp} and Lemma \ref{l:in-lab-mod}) $\LabMod(\nlab)$ is a partial stable model of $P$, and there is no partial stable model $\mathcal M' = \pair{T', F'}$ of $P$ such that $T' \subset T$ iff $\LabMod(\nlab)$ is a well-founded model of $P$;
    \item $\nlab$ is a preferred labelling of $\af_P$ iff $\nlab$ is a complete labelling of $\af_P$, and $\inn(\nlab)$ is maximal  (w.r.t. set inclusion) among all complete labellings of $\af_P$ iff (Theorem \ref{t:psmcomp} and Lemma \ref{l:in-lab-mod}) $\LabMod(\nlab)$ is a partial stable model of $P$, and there is no partial stable model $\mathcal M' = \pair{T', F'}$ of $P$ such that $T \subset T'$ iff $\LabMod(\nlab)$ is a regular model of $P$;
    \item $\nlab$ is a stable labelling of $\af_P$ iff $\nlab$ is a complete labelling of $\af_P$ such that $\undec(\nlab) = \emptyset$ iff (Theorem \ref{t:psmcomp}) $\LabMod(\nlab)$ is a partial stable model such that $\overline{T \cup F} = \emptyset$ iff $\LabMod(\nlab)$ is a stable model of $P$;
    \item $\nlab$ is a semi-stable labelling of $\af_P$ iff $\nlab$ is a complete labelling of $\af_P$, and $\undec(\nlab)$ is minimal  (w.r.t. set inclusion) among all complete labellings of $\af_P$ iff (Theorem \ref{t:psmcomp} and Lemma \ref{l:undec-lab-mod}) $\LabMod(\nlab)$ is a partial stable model of $P$, and there is no partial stable model $\mathcal M' = \pair{T', F'}$ of $P$ such that $\overline{T' \cap F'} \subset \overline{T \cup F}$ iff $\LabMod(\nlab)$ is an $L$-stable model of $P$.
\end{enumerate}
\hfill 
\end{proof}

\equivalenceII*

\begin{proof}
These results come from Theorems \ref{t:modarg} and \ref{t:equivalence}. \hfill 
\end{proof}

\subsection{Theorems and Proofs from Section \ref{s:setaflp}}

\inverseafsetaf*

\begin{proof}
Both results are immediate:

\begin{itemize}
    \item Proving that for any labelling $\nlab$ of $\af$, it holds $\ml(\lm(\nlab)) = \nlab$:
    
    Let $\lm(\nlab) = \pair{T, F}$. 
    
    \begin{itemize}
        \item $\nlab(a) = \inn$ $\Rightarrow$ $a \in T$ $\Rightarrow$ $\ml(\lm(\nlab))(a) = \inn$;
        \item $\nlab(a) = \out$ $\Rightarrow$ $a \in F$ $\Rightarrow$ $\ml(\lm(\nlab))(a) = \out$;
        \item $\nlab(a) = \undec$ $\Rightarrow$ $a \in \overline{T \cup F}$ $\Rightarrow$ $\ml(\lm(\nlab))(a) = \undec$.
    \end{itemize}
    
    \item Proving that for any interpretation $\inter$ of $P_\af$, it holds $\lm(\ml(\inter)) = \inter$.
    
      Let $\inter = \pair{T, F}$ be an interpretation of $P_\af$, and $\lm(\ml(\inter)) = \pair{T', F'}$. We will show $T = T'$ and $F = F'$: 
      
      \begin{itemize}
          \item $a \in T$ $\Rightarrow$ $\ml(\inter)(a) = \inn$ $\Rightarrow$ $a \in T'$; 
          \item $a \in F$ $\Rightarrow$ $\ml(\inter)(a) = \out$ $\Rightarrow$ $a \in F'$;
          \item $a \in \overline{T \cup F}$ $\Rightarrow$ $\ml(\inter)(a) = \undec$ $\Rightarrow$ $a \in \overline{T' \cup F'}$;
      \end{itemize}
\end{itemize}
\hfill 
\end{proof}

\comppsmsetaf*

\begin{proof}

\begin{enumerate}
    \item\label{i:comppsmfsetafI} Proving that if $\nlab$ is a complete labelling of $\af$, then $\lm(\nlab)$ is a partial stable model of $P_\af$: 

    Let $\lm(\nlab) = \pair{T, F}$ and $\Omega_{P_\af}(\lm(\nlab)) = \pair{T', F'}$. It suffices to show $\lm(\nlab)$ is a fixpoint of $\Omega_{P_\af}$: $T = T'$ and $F = F'$. For any argument $a \in \ar = \HB_{P_\af}$, there are three possibilities:

\begin{itemize}
    \item $a \in T$. Then $\nlab(a) = \inn$. From Definition \ref{d:labelling}, we know that for each $\mathcal B \in \att(a)$, it holds $\nlab(b) = \out$ for some $b \in \mathcal B$. It follows from Definition \ref{d:setaf-nlp} that there exists $V \in \mathcal V_a$ such that $\nlab(b) = \out$ for every $b \in V$. This means the fact $a \in \dfrac{P_\af}{\lm(\nlab)}$, i.e., $a \in T'$. 
    \item $a \in F$. Then $\nlab(a) = \out$. From Definition \ref{d:labelling}, we know that there exists $\mathcal B \in \att(a)$ such that $\nlab(b) = \inn$ for each $b \in \mathcal B$. It follows from Definition \ref{d:setaf-nlp} that for each $V \in \mathcal V_a$, there exists $b \in V$ such that $\nlab(b) = \inn$. This means that there exists no rule for $a$ in  $\dfrac{P_\af}{\lm(\nlab)}$, i.e., $a \in F'$. 
    \item $a \in \overline{T \cup F}$. Then $\nlab(a) = \undec$. From Definition \ref{d:labelling}, we know that (i) there exists $\mathcal B \in \att(a)$ such that $\nlab(b) \neq \out$ for each $b \in \mathcal B$, and (ii) for each $\mathcal B \in \att(a)$, it holds $\nlab(b) \neq \inn$ for some $b \in \mathcal B$. It follows from Definition \ref{d:setaf-nlp} that (i) there does not exist $V \in \mathcal V_a$ such that $\nlab(b) = \out$ for every $b \in V$, and (ii) there exists $V \in \mathcal V_a$ such that $\nlab(b) \neq \inn$ for each $b \in V$. This means (i) the fact $a \not\in \dfrac{P_\af}{\lm(\nlab)}$, and (ii) there exists rule for $a$ in $\dfrac{P_\af}{\lm(\nlab)}$. Thus $\body(r) = \mathbf{u}$ for any $r \in \dfrac{P_\af}{\lm(\nlab)}$ such that $\hd(r) = a$, i.e., $a \in \overline{T' \cup F'}$.
\end{itemize}

    \item\label{i:comppsmfsetafII} Proving that if $\mathcal M$ is a partial stable model of $P_\af$, then $\ml(\mathcal M)$ is a complete labelling of $\af$:

Let $\mathcal M = \pair{T, F}$ be a partial stable model of $P_\af$. Thus $\mathcal M$ is a fixpoint of $\Omega_{P_\af}$, i.e., $\Omega_{P_\af}(\mathcal M) = \mathcal M$. We now prove $\ml(\mathcal M)$ is a complete labelling of $\af$. For any  $a \in \HB_{P_\af} = \ar$, there are three possibilities:

\begin{itemize}
    \item $\ml(\mathcal M)(a) = \inn$. Then $a \in T$. As $\Omega_{P_\af}(\mathcal M) = \mathcal M$, the fact $a \in \dfrac{P_\af}{\mathcal M}$. This means that there exists a rule $a \leftarrow \naf b_1, \ldots, \naf b_n \in P_\af$ $(n \geq 0)$ such that $\set{b_1, \ldots, b_n} \subseteq F$. It follows from Definition \ref{d:setaf-nlp} that for each $\mathcal B \in \att(a)$, it holds $\ml(\mathcal M)(b) = \out$ for some $b \in \mathcal B$;
    
    \item $\ml(\mathcal M)(a) = \out$. Then $a \in F$. As $\Omega_{P_\af}(\mathcal M) = \mathcal M$, there exists no rule for $a$ in  $\dfrac{P_\af}{\mathcal M}$. This means that for every rule $a \leftarrow \naf b_1, \ldots, \naf b_n \in P_\af$ $(n \geq 0)$, there exists $b_i \in T$ ($1 \leq i \leq n$). It follows from Definition \ref{d:setaf-nlp} that there exists $\mathcal B \in \att(a)$ such that $\ml(\mathcal M)(b) = \inn$ for each $b \in \mathcal B$;
    
    \item $\ml(\mathcal M)(a) = \undec$. Then $a \in \overline{T \cup F}$. As $\Omega_{P_\af}(\mathcal M) = \mathcal M$, the fact $a \not\in \dfrac{P_\af}{\mathcal M}$, but there exists a rule $r$ in $\dfrac{P_\af}{\mathcal M}$ such that $\hd(r) = a$ and $\body(r) = \mathbf{u}$. This means that (i)  for each rule $a \leftarrow \naf b_1, \ldots, \naf b_n \in P_\af$ $(n \geq 0)$, it holds $\set{b_1, \ldots, b_n} \not\subseteq F$, and (ii) there exists a rule $a \leftarrow \naf b_1, \ldots, \naf b_n \in P_\af$ $(n \geq 0)$ such that $\set{b_1, \ldots, b_n} \cap T = \emptyset$. It follows from Definition \ref{d:setaf-nlp} that (i) there exists $\mathcal B \in \att(a)$ such that $\ml(\mathcal M)(b) \neq \out$ for each $b \in \mathcal B$, and (ii) for each $\mathcal B \in \att(a)$, it holds $\ml(\mathcal M)(b) \neq \inn$ for some $b \in \mathcal B$.
\end{itemize}

Hence, $\ml(\mathcal M)$ is a complete labelling of $\af$.
    
    \item Proving that if $\lm(\nlab)$ is a partial stable model of $P_\af$, then $\nlab$ is a complete labelling of $\af$: 

$\lm(\nlab)$ is a partial stable model of $P_\af$ $\Rightarrow$ according to item \ref{i:comppsmfsetafII} above, $\ml(\lm(\nlab))$ is a complete labelling of $\af$ $\Rightarrow$ (Theorem \ref{t:inverse-af-setaf}) $\nlab$ is a complete labelling of $\af$.
    
    \item Proving that if $\ml(\mathcal M)$ is a complete labelling of $\af$, then $\mathcal M$ is a partial stable model of $P_\af$: 

 $\ml(\mathcal M)$ is complete labelling of $\af$ $\Rightarrow$ according to item \ref{i:comppsmfsetafI} above, $\lm(\ml(\mathcal M))$ is a partial stable model of $P_\af$ $\Rightarrow$ (Theorem \ref{t:inverse-af-setaf}) $\mathcal M$ is a partial stable model of $P_\af$.
\end{enumerate}
\hfill 
\end{proof}

\equivalencesetaf*

\begin{proof}
Let $\nlab$ be an argument labelling of $\af$. Recall that $\lm(\nlab) = \pair{\inn(\nlab), \out(\nlab)}$. The proof is straightforward:

\begin{enumerate}
    \item $\nlab$ is a grounded labelling of $\af$ iff $\nlab$ is a complete labelling of $\af$ and $\inn(\nlab)$ is minimal  (w.r.t. set inclusion) among all complete labellings of $\af$ iff (Theorem \ref{t:comppsmsetaf}) $\lm(\nlab)$ is a partial stable model of $P_\af$ and there is no partial stable model $\mathcal M' = \pair{T', F'}$ of $P_\af$ such that $T' \subset \inn(\nlab)$ iff $\lm(\nlab)$ is a well-founded model of $P_\af$;
    \item $\nlab$ is a preferred labelling of $\af$ iff $\nlab$ is a complete labelling of $\af$ and $\inn(\nlab)$ is maximal  (w.r.t. set inclusion) among all complete labellings of $\af$ iff (Theorem \ref{t:comppsmsetaf}) $\lm(\nlab)$ is a partial stable model of $P_\af$ and there is no partial stable model $\mathcal M' = \pair{T', F'}$ of $P$ such that $\inn(\nlab) \subset T'$ iff $\lm(\nlab)$ is a regular model of $P_\af$;
    \item $\nlab$ is a stable labelling of $\af$ iff $\nlab$ is a complete labelling of $\af$ such that $\undec(\nlab) = \emptyset$ iff (Theorem \ref{t:comppsmsetaf}) $\lm(\nlab)$ is a partial stable model of $P_\af$ such that $\overline{\inn(\nlab) \cup \out(\nlab)} = \emptyset$ iff $\lm(\nlab)$ is a stable model of $P_\af$;
    \item $\nlab$ is a semi-stable labelling of $\af$ iff $\nlab$ is a complete labelling of $\af$ and $\undec(\nlab)$ is minimal  (w.r.t. set inclusion) among all complete labellings of $\af$ iff (Theorem \ref{t:comppsmsetaf}) $\lm(\nlab)$ is a partial stable model of $P_\af$ and there is no partial stable model $\mathcal M' = \pair{T', F'}$ of $P_\af$ such that $\overline{T' \cup F'} \subset \overline{\inn(\nlab) \cup \out(\nlab)}$ iff $\lm(\nlab)$ is an $L$-stable model of $P_\af$.
\end{enumerate}
 \hfill  
\end{proof}

\setafequivalenceII*
\begin{proof}
These results come from Theorems \ref{t:inverse-af-setaf} and \ref{t:equivalencesetaf}. \hfill 
\end{proof}

\subsection{Theorems and Proofs from Section \ref{s:slp-setaf}}
    
\slp*

\begin{proof}
It follows that

\begin{enumerate}
    \item Each rule in $P_\af$ has the form $a \leftarrow \naf b_1, \ldots, \naf b_n$; 
    \item for each rule $a \leftarrow \naf b_1, \ldots, \naf b_n \in P_\af$, if $b \in \set{b_1, \ldots, b_n}$, there exists $(\mathcal B, a) \in \att$ such that $b \in \mathcal B$, i.e,. $b \in \ar_P$. Then there exists a rule $r \in P_\af$ such that $b = \hd(r)$. This suffices to guarantee $\HB_{P_\af} = \set{\hd(r) \mid r \in P_\af}$;
    \item A rule $a \leftarrow \naf b_1, \ldots, \naf b_n \in P_\af$ iff there exists a minimal set (w.r.t. set inclusion) $V = \set{b_1, \ldots, b_n}$ such that for each $\mathcal B \in \att(a)$, there exists $b \in \mathcal B \cap V$. This means there exists no rule $a \leftarrow \naf c_1, \ldots, \naf c_{n'} \in P_\af$ such that $\set{c_1, \ldots, c_{n'}} \subset \set{b_1, \ldots, b_n}$.
\end{enumerate}

Hence, $P_\af$ is an $\rfalp$. \hfill 
\end{proof}

\begin{lem}\label{l:hp-arp}
Let $P$ be an $\rfalp$, $\head_P = \set{ \mathit{head}(r) \mid r \in P}$ and $\af_P = (\ar_P, \att_P)$ its corresponding $\setaf$.  It holds $\head_P = \ar_P$.
\end{lem}

\begin{proof}
The result is straightforward: $c \in \head_P$ iff there exists a rule $c \leftarrow \naf b_1, \ldots, \naf b_n \in P$ ($n \geq 0$) iff $c \in \ar_P$ (Definition \ref{d:argument}). \hfill 
\end{proof}

\inversesetafslp*

\begin{proof}
Let $\af = (\ar, \att)$ be a $\setaf$ with $\ar = \set{a_1, \ldots, a_n}$ and for each $a_i \in \ar$, we define $R_i = \set{r \in P_\af \mid \mathit{head}(r) = a_i}$, i.e., $P_\af = R_1 \cup R_2 \cup \cdots \cup R_n$. It follows from Proposition \ref{p:slp} and Lemma \ref{l:hp-arp} that $\af_{P_\af} = (\ar_{P_\af}, \att_{P_\af})$ with $\ar_{P_\af} = \set{a_1, \ldots, a_n} = \ar$. It remains to prove that $\att = \att_{P_\af}$:

$(\mathcal B, a_j) \in \att$ iff $(\mathcal B, a_j) \in \att$ and  there exists no $\mathcal B' \subset \mathcal B$ such that $(\mathcal B', a_j) \in \att$ iff $\mathcal B$ is a minimal set (w. r. t. set inclusion) in which for each rule $r \in R_j$, there exists $b \in \mathcal B$ such that $\naf b \in \body^-(r)$ iff $\mathcal B$ is a minimal set (w. r. t. set inclusion) in which for each $V \in \Vul(a_j)$, there exists $b \in \mathcal B \cap V$ iff $(\mathcal B, a_j) \in \att_{P_\af}$.
\hfill
\end{proof}

\begin{lem}\label{l:inverse-slp-setaf}
Let $P$ be an $\rfalp$, $\af_P = (\ar_P, \att_P)$ the corresponding $\setaf$ and $c \in \ar_P$. If $\set{a_1, \ldots, a_n}$ is a minimal set such that for each $\mathcal B \in \att_P(c)$, there exists $a_i \in \mathcal B$ ($1 \leq i \leq n$), then $c \leftarrow \naf a_1, \ldots, \naf a_n \in P$. 
\end{lem}

\begin{proof}
As for each $\mathcal B \in \att_P(c)$, there exists $a_i \in \mathcal B$ ($1 \leq i \leq n$), it follows from Definition \ref{d:attack} that there exists $V \in \Vul(c)$ such that $V \subseteq \set{a_1, \ldots, a_n}$. Note that for each $\mathcal B \in \att_P(c)$, there exists $b \in V \cap \mathcal B$. As $\set{a_1, \ldots, a_n}$ is a minimal set with this property, it holds $V = \set{a_1, \ldots, a_n}$. Then (Definition \ref{d:argument}) $c \leftarrow \naf a_1, \ldots, \naf a_n \in P$. \hfill 
\end{proof}

\inverseslpsetaf*

\begin{proof}

Let $P$ be an $\rfalp$ with $\HB_P = \set{a_1, \ldots, a_n}$, and $\af_P = (\ar_P, \att_P)$ the corresponding $\setaf$. For each $a_i \in \HB_P$ ($1 \leq i \leq n$), we define $R_i = \set{r \in P_\af \mid \mathit{head}(r) = a_i}$. 
It follows that $\ar_P = \set{a_1, \ldots, a_n}$. Hence, $\HB_{P_{\af_P}} = \set{a_1, \ldots, a_n}$. We will prove $P = P_{\af_P}$:

\begin{itemize}
    \item If $a_i \leftarrow \naf a_{i_1}, \ldots, \naf a_{i_m} \in P$. then  $a_i \in \ar_P$ and $\set{a_{i_1}, \ldots, a_{i_m}}$ is a minimal set (w.r.t. set inclusion) such that for each $\mathcal B \in \att_P(a_i)$, there exists $a_{i_k} \in \mathcal B$ ($k \in \set{1, \ldots, m}$). This implies (Definition \ref{d:setaf-nlp}) $a_i \leftarrow \naf a_{i_1}, \ldots, \naf a_{i_m} \in P_{\af_P}$.
    \item If $a_i \leftarrow \naf a_{i_1}, \ldots, \naf a_{i_m} \in P_{\af_P}$, then (Definition \ref{d:setaf-nlp}) $\set{a_{i_1}, \ldots, a_{i_m}}$ is a minimal set (w.r.t. set inclusion) such that for each $\mathcal B \in \att_P(a_i)$, there exists $a_{i_k} \in \mathcal B$ ($k \in \set{1, \ldots, m}$). Thus (Lemma \ref{l:inverse-slp-setaf}) $a_i \leftarrow \naf a_{i_1}, \ldots, \naf a_{i_m} \in P$. 
\end{itemize}
\hfill 
\end{proof}

\subsection{Theorems and Proofs from Section \ref{s:expressiveness}}

\terminates*

\begin{proof}

Let $P_1 \mapsto_\utpm P_2 \mapsto_\utpm \cdots \mapsto_\utpm P_k \mapsto_\utpm \cdots \mapsto_\utpm P_{k'} \mapsto_\utpm \cdots$ be a fair sequence of $\mapsto_\utpm$. This fairness condition implies that for every atom $a$, there exists a natural number $k$ such that for each $\nlp$ $P_i$ with $i > k$ in the sequence of $\mapsto_\utpm$ above, it holds $a \not\in \body^+(r)$ for each $r \in P_i$. As each $\nlp$ is a finite set of rules, from some natural number $k'$ on, $\body^+(r) = \emptyset$ for any $r \in P_{k'}$. Then for each $k'' \geq k'$, $\mapsto_U$ and $\mapsto_T$ cannot be applied in $P_{k''}$. It remains the program transformations $\mapsto_P$ and $\mapsto_M$. For each of these $P_{k''}$, there are two possibilities:
\begin{itemize}
\item $\mapsto_M$ strictly
decreases the number of rules of $P_{k''}$ or
\item $\mapsto_P$ strictly
decreases the number of negative literals in $\body^-(r)$ for some $r \in P_{k''}$. 
\end{itemize}

It follows that the successive application of $\mapsto_M$ or $\mapsto_P$ in these $P_{k''}$s will eventually lead to an irreducible $\nlp$. \hfill
\end{proof}

\irreducible*

\begin{proof}
A simple method to obtain a fair sequence of program transformations with respect to $\mapsto_\utpm$ is to apply $\mapsto_U$ to a rule $r$ only if $\mapsto_T$ is not applicable to $r$ and to ensure that whenever $\mapsto_U$ has been applied to get rid of an occurrence of an atom $a$, then all such occurrences of $a$ (in other rules of the same program) have also been removed before applying $\mapsto_U$ to another occurrence of an atom $b \neq a$.

As for any $\nlp$ $P$, it is always possible to build such a fair sequence of program transformations with respect to $\mapsto_\utpm$, we obtain from Theorem \ref{t:terminates} that there exists an irreducible $\nlp$ $P^*$ such that $P \mapsto^*_\utpm P^*$. \hfill
\end{proof}

\utpmsetaf*

\begin{proof}
To prove it by contradiction, suppose $P^*$ is not an $\rfalp$. There are three possibilities:

\begin{itemize}
    \item A rule $c \leftarrow a_1, \ldots, a_m, \naf b_1, \ldots, \naf b_n \in P^*$ with $m \geq 1$ and $n \geq 0$. Then
    \begin{itemize}
       \item The program transformation $\mapsto_U$ (unfolding) can be applied.
        \item If $c \in \set{a_1, \ldots, a_m}$, the program transformation $\mapsto_T$ (elimination of tautologies) can be applied.
    \end{itemize}
    \item A rule $c  \leftarrow \naf b_1, \ldots, \naf b_n \in P^*$, but there exists $b \in \set{b_1, \ldots, b_n}$ such that $b \not\in \set{\hd(r) \mid r \in P^*}$. Then the program transformation $\mapsto_P$ (positive reduction) can be applied. 
    \item A rule $c \leftarrow \naf b_1, \ldots, \naf b_n \in P^*$ and there is a rule $c \leftarrow \naf c_1, \ldots, \naf c_p \in P^*$ such that $\set{c_1, \ldots, c_p} \subset \set{b_1, \ldots, b_n}$. Then the program transformation $\mapsto_M$ (elimination of non-minimal rules) can be applied.
\end{itemize}
It is absurd as in each case, there is still a program transformation to be applied. \hfill
\end{proof}

\setafutpm*

\begin{proof}
Let $P$ be an $\rfalp$. It holds
\begin{itemize}
    \item The program transformations $\mapsto_U$ and $\mapsto_T$ cannot be applied as they require a rule $c \leftarrow a_1, \ldots, a_m, \naf b_1, \ldots, \naf b_n$ in $P$ with $m \geq 1$.
    \item The program transformation $\mapsto_P$ cannot be applied as it requires a rule $c \leftarrow a_1, \ldots, a_m, \naf b, \naf b_1, \ldots, \naf b_n$ in $P$ such that $b \not\in \set{\hd(r) \mid r \in P}$, but $\set{\hd(r) \mid r \in P} = \HB_P$.
    \item The program transformation $\mapsto_M$ cannot be applied as it requires two distinct rules $r$ and $r'$ in $P$ such that $\hd(r) = \hd(r')$ and $\body^-(r') \subset \body^-(r)$. 
\end{itemize}
\hfill
\end{proof}

\psmT*

\begin{proof}
Let $P_2 = P_1 - \set{r}$ and $\hd(r) \in \body^+(r)$. We have to show for any interpretation $\mathcal M = \pair{T, F }$, it holds $\mathcal M$ is a partial stable model of $P_1$ iff $\mathcal M$ is a partial stable model of $P_2$; we distinguish two cases:

\begin{itemize}
    \item $\set{a \mid \naf a \in \body^-(r)} \cap T \neq \emptyset$: Then $\frac{P_1}{\mathcal M} = \frac{P_2}{\mathcal M}$. This trivially implies that $\mathcal M$ is a partial stable model of $P_1$ iff it is a partial stable model of $P_2$.
    \item $\set{a \mid \naf a \in \body^-(r)} \cap T = \emptyset$: Then it is clear $\frac{P_1}{\mathcal M} \mapsto_T \frac{P_2}{\mathcal M}$. As both $\frac{P_1}{\mathcal M}$ and $\frac{P_2}{\mathcal M}$ are positive programs, according to Lemma \ref{l:definitex}, it holds $\mathcal M$ is the least model of $\frac{P_1}{\mathcal M}$ iff $\mathcal M$ is the least model of $\frac{P_2}{\mathcal M}$. Hence,  $\mathcal M$ is a partial stable model of $P_1$ iff it is a partial stable model of $P_2$.   
\end{itemize}
\hfill
\end{proof}

\psmP*

\begin{proof}

Let 
\begin{align*}
    P_2 = & P_1 - \set{c \leftarrow a_1, \ldots, a_m, \naf b, \naf b_1, \ldots, \naf b_n}\\ 
    & \hspace{4.625em} \cup \set{c \leftarrow a_1, \ldots, a_m, \naf b_1, \ldots, \naf b_n}
\end{align*}
such that $r$ is the rule $c \leftarrow a_1, \ldots, a_m, \naf b, \naf b_1, \ldots, \naf b_n \in P_1$ and $b \not\in \set{\hd(r') \mid r' \in P_1}$. We have to show that for any interpretation $\mathcal M = \pair{T, F}$, it holds $\mathcal M$ is a partial stable model of $P_1$ iff $\mathcal M$ is a partial stable model of $P_2$; we distinguish two cases:

\begin{itemize}
\item $\left(\set{a \mid \naf a \in \body^-(r)} - \set{b}\right) \cap T \neq \emptyset$ or $b \in F$: Then $\frac{P_1}{\mathcal M} = \frac{P_2}{\mathcal M}$. This trivially implies that $\mathcal M$ is a partial stable model of $P_1$ iff it is a partial stable model of $P_2$.
\item $\left(\set{a \mid \naf a \in \body^-(r)} - \set{b}\right) \cap T = \emptyset$ and $b \not\in F$. Let $\pair{T_1, F_1}$ and $\pair{T_2, F_2}$ be respectively the least models of $\frac{P_1}{\mathcal M}$ and $\frac{P_2}{\mathcal M}$. As $b \not\in \set{\hd(r') \mid r' \in P_1}$, it is clear that $b \in F_1$ and $b \in F_2$. Given that $b \not\in F$, we obtain $\mathcal M = \pair{T, F}$ is different from both $\pair{T_1, F_1}$ and $\pair{T_2, F_2}$. Hence, $\mathcal M$ is neither a partial stable model of $P_1$ nor of $P_2$.  This implies that $\mathcal M$ is a partial stable model of $P_1$ iff it is a partial stable model of $P_2$.
\end{itemize}
\hfill 
\end{proof}

\psmM*

\begin{proof}
Suppose that there are two distinct rules $r$ and $r'$ in $P_1$ such that $\hd(r) = \hd(r')$, $\body^+(r') \subseteq \body^+(r)$, $\body^-(r') \subseteq \body^-(r)$ and $P_2 = P_1 - \set{r}$. We have to show that for any interpretation $\mathcal M = \pair{T, F}$, it holds that $\mathcal M$ is a partial stable model of $P_1$ iff $\mathcal M$ is a partial stable model of $P_2$; we distinguish two cases:

\begin{itemize}
\item $\set{a \mid \naf a \in \body^-(r)} \cap T \neq \emptyset$ or ($\set{a \mid \naf a \in \body^-(r)} \cap T = \emptyset$ and $\body^+(r) = \body^+(r')$): Then $\frac{P_1}{\mathcal M} = \frac{P_2}{\mathcal M}$. This trivially implies that $\mathcal M$ is a partial stable model of $P_1$ iff it is a partial stable model of $P_2$.

\item $\set{a \mid \naf a \in \body^-(r)} \cap T = \emptyset$ and $\body^+(r') \subset \body^+(r)$: Then it is clear that $\frac{P_1}{\mathcal M} \mapsto_M \frac{P_2}{\mathcal M}$. As both $\frac{P_1}{\mathcal M}$ and $\frac{P_2}{\mathcal M}$ are positive programs, according to Lemma \ref{l:definitex}, it holds that $\mathcal M$ is the least model of $\frac{P_1}{\mathcal M}$ iff $\mathcal M$ is least model $\frac{P_2}{\mathcal M}$. Hence,  $\mathcal M$ is a partial stable model of $P_1$ iff it is a partial stable model of $P_2$. 
\end{itemize}
\hfill 
\end{proof}

\psmUTPM*

\begin{proof}
If $P \mapsto^*_\utpm P^*$, then there exists a finite sequence of program transformations $P = P_1 \mapsto_\utpm \cdots \mapsto_\utpm P_n = P^*$. According to Theorems \ref{t:psm-U}, \ref{t:psm-T}, \ref{t:psm-P} and \ref{t:psm-M}, $\mathcal M$ is a partial stable model of $P_i$ iff $\mathcal M$ is a partial stable model of $P_{i + 1}$ with $1 \leq i < n$. Thus by transitivity, $\mathcal M$ is a partial stable model of $P$ iff $\mathcal M$ is a partial stable model of $P^*$. 

\hfill 
\end{proof}

\equivUTPM*

\begin{proof}
As $P$ and $P^*$ share the same set of partial stable models (Theorem \ref{t:psm-UTPMF}), the result is straightforward. \hfill
\end{proof}

\lpequivsetaf*

\begin{proof}
From Theorem \ref{t:irreducible}, we know that for any $\nlp$ $P$, there exists an irreducible $\nlp$ $P^*$ such that $P \mapsto^*_\utpm P^*$. From Theorem \ref{t:utpmf-setaf}, we obtain $P^*$ is an $\rfalp$. Besides, from Theorem \ref{t:psm-UTPMF} and Corollary \ref{c:equiv-UTPMF}, we infer $\mathcal M$ is a partial stable, well-founded, regular, stable, $L$-stable model of $P$ iff $\mathcal M$ is respectively a partial stable, well-founded, regular, stable, $L$-stable model of $P^*$. \hfill
\end{proof}

\lpslp*

\begin{proof}
We have 
\begin{itemize}
    \item For any $\nlp$ $P$, there exists an $\rfalp$ $P^*$ such that $\mathcal M$ is a partial stable, well-founded, regular, stable, $L$-stable model of $P$ iff $\mathcal M$ is respectively a partial stable, well-founded, regular, stable, $L$-stable model of $P^*$ (Corollary \ref{c:lpequivsetaf}).
    \item Obviously, any $\rfalp$ is an $\nlp$.
\end{itemize}

Hence, $\nlp$s and $\rfalp$s have the same expressiveness for partial stable, well-founded, regular, stable and $L$-stable semantics. \hfill
\end{proof}

\begin{lem}\label{l:invariant-u}
Let $P_1$ and $P_2$ be $\nlp$s such that $P_1 \mapsto_U P_2$. It holds that $\af_{P_1} = \af_{P_2}$.
\end{lem}

\begin{proof}
Let $P_1$ and $P_2$ be $\nlp$s such that
\begin{align*}
    P_2 = & P_1 - \set{c \leftarrow a, a_1, \ldots, a_m, \naf b_1, \ldots, \naf b_n}\\ 
    & \hspace{5em} \cup \{c \leftarrow a_1', \ldots, a_p', a_1, \ldots, a_m, \naf b_1', \ldots, \naf b_q', \naf b_1, \ldots, \naf b_n \mid \\
    & \hspace{6.675em} a \leftarrow a_1', \ldots, a_p', \naf b_1', \ldots, \naf b_q' \in P_1 \},
\end{align*}
\noindent $\af_{P_1} = (\ar_{P_1},\att_{P_1})$ and $\af_{P_2} = (\ar_{P_2},\att_{P_2})$. Note that 
\begin{itemize}
\item For each statement $s \in \mathfrak S_{P_1}$, there exists $s' \in \mathfrak S_{P_2}$ such that $\Conc(s) = \Conc(s')$, and $\Vul(s) = \Vul(s')$. 
\item For each statement $s' \in \mathfrak S_{P_2}$, there exists $s \in \mathfrak S_{P_1}$ such that $\Conc(s') = \Conc(s)$, and $\Vul(s') = \Vul(s)$. 
\end{itemize}
Hence, $\ar_{P_1} = \ar_{P_2}$, and $\att_{P_1} = \att_{P_2}$. \hfill 
\end{proof}

\begin{lem}\label{l:invariant-t}
Let $P_1$ and $P_2$ be $\nlp$s such that $P_1 \mapsto_T P_2$. It holds that $\af_{P_1} = \af_{P_2}$.
\end{lem}

\begin{proof}
Let $P_2 = P_1 - \set{r}$, where there exists a rule $r \in P_1$ such that $\hd(r) \in \body^+(r)$. In addition, let $\af_{P_1} = (\ar_{P_1},\att_{P_1})$ and $\af_{P_2} = (\ar_{P_2},\att_{P_2})$. Note that

\begin{itemize}
\item For each statement $s \in \mathfrak S_{P_1}$, there exists $s' \in \mathfrak S_{P_2}$ such that $\Conc(s) = \Conc(s')$, and for each $V \in \Vul(s)$, there exists $V' \in \Vul(s) \cap \Vul(s')$ such that $V' \subseteq V$. 
\item For each statement $s' \in \mathfrak S_{P_2}$, there exists $s \in \mathfrak S_{P_1}$ such that $\Conc(s') = \Conc(s)$, and $\Vul(s') = \Vul(s)$. 
\end{itemize}
Hence, $\ar_{P_1} = \ar_{P_2}$, and for each $c \in \ar_{P_1}$, $V$ is a minimal set (w.r.t. set inclusion) in $\Vul_{P_1}(c)$ iff $V$ is a minimal set (w.r.t. set inclusion) in $\Vul_{P_2}(c)$; it holds that $\att_{P_1} = \att_{P_2}$. \hfill
\end{proof}

\begin{lem}\label{l:invariant-p}
Let $P_1$ and $P_2$ be $\nlp$s such that $P_1 \mapsto_P P_2$. It holds that $\af_{P_1} = \af_{P_2}$.
\end{lem}

\begin{proof}
Let $c \leftarrow a_1, \ldots, a_m, \naf b, \naf b_1, \ldots, \naf b_n \in P_1$ be a rule such that $b \not\in \set{\hd(r) \mid r \in P_1}$, 
\begin{align*}
    P_2 = & (P_1 - \set{c \leftarrow a_1, \ldots, a_m, \naf b, \naf b_1, \ldots, \naf b_n})\\ 
    & \hspace{4.475em} \cup \set{c \leftarrow a_1, \ldots, a_m, \naf b_1, \ldots, \naf b_n},
\end{align*}
\noindent $\af_{P_1} = (\ar_{P_1},\att_{P_1})$ and $\af_{P_2} = (\ar_{P_2},\att_{P_2})$. Note that 
\begin{itemize}
\item For each statement $s \in \mathfrak S_{P_1}$, there exists $s' \in \mathfrak S_{P_2}$ such that $\Conc(s) = \Conc(s')$, and $\Vul(s) = \set{V \mid \exists V' \in \Vul(s') \textit{ such that } V = V' \textit{ or } V = V' \cup \set{b} }$. 
\item For each statement $s' \in \mathfrak S_{P_2}$, there exists $s \in \mathfrak S_{P_1}$ such that $\Conc(s') = \Conc(s)$, and $\Vul(s') = \set{V' \mid \exists V \in \Vul(s) \textit{ such that } V' = V \textit{ or } V' = V - \set{b} }$. 
\end{itemize}

Hence, $\ar_{P_1} = \ar_{P_2}$, and as $b \not\in \ar_{P_1} \cup \ar_{P_2}$, it holds that $\att_{P_1} = \att_{P_2}$. \hfill 
\end{proof}

\begin{lem}\label{l:invariant-m}
Let $P_1$ and $P_2$ be $\nlp$s such that $P_1 \mapsto_M P_2$. It holds that $\af_{P_1} = \af_{P_2}$.
\end{lem}

\begin{proof}
Let $P_2 = P_1 - \set{r}$, where there are two distinct rules $r$ and $r'$ in $P_1$ such that $\hd(r) = \hd(r')$, $\body^+(r') \subseteq \body^+(r)$, $\body^-(r') \subseteq \body^-(r)$. In addition, let $\af_{P_1} = (\ar_{P_1},\att_{P_1})$ and $\af_{P_2} = (\ar_{P_2},\att_{P_2})$. Note that

\begin{itemize}
\item For each statement $s \in \mathfrak S_{P_1}$, there exists $s' \in \mathfrak S_{P_2}$ such that $\Conc(s) = \Conc(s')$, and for each $V \in \Vul(s)$, there exists $V' \in \Vul(s) \cap \Vul(s')$ such that $V' \subseteq V$. 
\item For each statement $s' \in \mathfrak S_{P_2}$, there exists $s \in \mathfrak S_{P_1}$ such that $\Conc(s') = \Conc(s)$, and $\Vul(s') = \Vul(s)$. 
\end{itemize}
Hence, $\ar_{P_1} = \ar_{P_2}$, and for each $c \in \ar_{P_1}$, $V$ is a minimal set (w.r.t. set inclusion) in $\Vul_{P_1}(c)$ iff $V$ is a minimal set (w.r.t. set inclusion) in $\Vul_{P_2}(c)$; it holds that $\att_{P_1} = \att_{P_2}$. \hfill
\end{proof}

\invariant*

\begin{proof}
It follows straightforwardly from Lemmas \ref{l:invariant-u}, \ref{l:invariant-t}, \ref{l:invariant-p} and \ref{l:invariant-m}.
\hfill
\end{proof}

\confluent*

\begin{proof}
From Theorem \ref{t:invariant}, we know that $\af_P = \af_{P'} = \af_{P''}$. Thus  $P_{\af_{P'}} = P_{\af_{P''}}$. As $P'$ and $P''$ are $\rfalp$s (Theorem \ref{t:utpmf-setaf}), it holds (Theorem \ref{t:inverse-slp-setaf}) that $P' = P_{\af_{P'}} = P_{\af_{P''}} = P''$. \hfill 
\end{proof}


\begin{thebibliography}{references}

\bibitem[Alc{\^a}ntara and S{\'a}, 2021]{alcantara2021equivalence}
{\sc Alc{\^a}ntara, J.} {\sc and} {\sc S{\'a}, S.} 2021.
\newblock Equivalence results between {SETAF} and attacking abstract dialectical. frameworks.
\newblock In {\em Proceedings NMR}, pp. 139--48.

\bibitem[Alc{\^a}ntara et~al., 2019]{alcantara2019equivalence}
{\sc Alc{\^a}ntara, J.}, {\sc S{\'a}, S.}, {\sc and} {\sc Acosta-Guadarrama, J.} 2019.
\newblock On the equivalence between abstract dialectical frameworks and logic programs.
\newblock {\em Theory and Practice of Logic Programming}, {\it 19}, 5--6, 941--956.

\bibitem[Aravindan and Minh, 1995]{aravindan1995correctness}
{\sc Aravindan, C.} {\sc and} {\sc Minh, D.~P.} 1995.
\newblock On the correctness of unfold/fold transformation of normal and extended logic programs.
\newblock {\em The Journal of Logic Programming}, {\it 24}, 3, 201--217.

\bibitem[Beirlaen et~al., 2018]{beirlaen2018argument}
{\sc Beirlaen, M.}, {\sc Heyninck, J.}, {\sc Pardo, P.}, {\sc and} {\sc Stra{\ss}er, C.} 2018.
\newblock Argument strength in formal argumentation.
\newblock {\em FLAP}, {\it 5}, 3, 629--676.

\bibitem[Bondarenko et~al., 1997]{bondarenko1997abstract}
{\sc Bondarenko, A.}, {\sc Dung, P.~M.}, {\sc Kowalski, R.~A.}, {\sc and} {\sc Toni, F.} 1997.
\newblock An abstract, argumentation-theoretic approach to default reasoning.
\newblock {\em Art. Intelligence}, {\it 93}, 1-2, 63--101.

\bibitem[Brass and Dix, 1994]{brass1994disjunctive}
{\sc Brass, S.} {\sc and} {\sc Dix, J.} 1994
\newblock A disjunctive semantics based on unfolding and bottom-up evaluation.
\newblock In {\em Innovationen bei Rechen-und Kommunikationssystemen (IFIP-Congress, Workshop FG2: Disjunctive Logic Programming and Disjunctive Databases)}, pp. 83--91. Springer Berlin Heidelberg.

\bibitem[Brass and Dix, 1995]{brass1995disjunctive}
{\sc Brass, S.} {\sc and} {\sc Dix, J.} 1995.
\newblock Disjunctive semantics based upon partial and bottom-up evaluation.
\newblock In {\em ICLP}, pp. 199--213.

\bibitem[Brass and Dix, 1997]{brass1997characterizations}
{\sc Brass, S.} {\sc and} {\sc Dix, J.} 1997.
\newblock Characterizations of the disjunctive stable semantics by partial evaluation.
\newblock {\em The Journal of Logic Programming}, {\it 32}, 3, 207--228.

\bibitem[Brass and Dix, 1998]{brass1998characterizations}
{\sc Brass, S.} {\sc and} {\sc Dix, J.} 1998.
\newblock Characterizations of the disjunctive well-founded semantics: confluent calculi and iterated gcwa.
\newblock {\em Journal of automated reasoning}, {\it 20}, 143--165.

\bibitem[Brass and Dix, 1999]{brass1999semantics}
{\sc Brass, S.} {\sc and} {\sc Dix, J.} 1999.
\newblock Semantics of (disjunctive) logic programs based on partial evaluation.
\newblock {\em The Journal of Logic Programming}, {\it 40}, 1, 1--46.

\bibitem[Brewka et~al., 2013]{brewka2013abstract}
{\sc Brewka, G.}, {\sc Ellmauthaler, S.}, {\sc Strass, H.}, {\sc Wallner, J.~P.}, {\sc and} {\sc Woltran, S.} 2013.
\newblock Abstract dialectical frameworks revisited.
\newblock In {\em Proceedings of the Twenty-Third international joint conference on Artificial Intelligence}, pp. 803--809.

\bibitem[Brewka and Woltran, 2010]{brewka2010abstract}
{\sc Brewka, G.} {\sc and} {\sc Woltran, S.} 2010.
\newblock Abstract dialectical frameworks.
\newblock In {\em Proceedings of the Twelfth International Conference on Principles of Knowledge Representation and Reasoning}, pp. 102--111.

\bibitem[Caminada, 2006]{caminada2006semi}
{\sc Caminada, M.} 2006.
\newblock Semi-stable semantics.
\newblock {\em 1st International Conference on Computational Models of Argument (COMMA)}, {\it 144}, 121--130.

\bibitem[Caminada and Amgoud, 2005]{caminada2005axiomatic}
{\sc Caminada, M.} {\sc and} {\sc Amgoud, L.} 2005.
\newblock An axiomatic account of formal argumentation.
\newblock In {\em AAAI}, volume~6, pp. 608--613.

\bibitem[Caminada and Amgoud, 2007]{caminada2007evaluation}
{\sc Caminada, M.} {\sc and} {\sc Amgoud, L.} 2007.
\newblock On the evaluation of argumentation formalisms.
\newblock {\em Artificial Intelligence}, {\it 171}, 5-6, 286--310.

\bibitem[Caminada et~al., 2022]{caminada2022comparing}
{\sc Caminada, M.}, {\sc Harikrishnan, S.}, {\sc and} {\sc S{\'a}, S.} 2022.
\newblock Comparing logic programming and formal argumentation; the case of ideal and eager semantics.
\newblock {\em Argument \& Computation}, {\it 13}, 1, 93--120.

\bibitem[Caminada et~al., 2015a]{caminada2015difference}
{\sc Caminada, M.}, {\sc S{\'a}, S.}, {\sc Alc{\^a}ntara, J.}, {\sc and} {\sc Dvo{\v{r}}{\'a}k, W.} 2015a.
\newblock On the difference between assumption-based argumentation and abstract argumentation.
\newblock {\em IFCoLog Journal of Logic and its Applications}, {\it 2}a, 1, 15--34.

\bibitem[Caminada et~al., 2015b]{caminada15equivalence}
{\sc Caminada, M.}, {\sc S{\'a}, S.}, {\sc Alc{\^a}ntara, J.}, {\sc and} {\sc Dvo{\v{r}}{\'a}k, W.} 2015b.
\newblock On the equivalence between logic programming semantics and argumentation semantics.
\newblock {\em International Journal of Approximate Reasoning}, {\it 58}b, 87--111.

\bibitem[Caminada and Schulz, 2017]{caminada2017equivalence}
{\sc Caminada, M.} {\sc and} {\sc Schulz, C.} 2017.
\newblock On the equivalence between assumption-based argumentation and logic programming.
\newblock {\em Journal of Artificial Intelligence Research}, {\it 60}, 779--825.

\bibitem[Caminada and Gabbay, 2009]{caminada2009logical}
{\sc Caminada, M.~W.} {\sc and} {\sc Gabbay, D.~M.} 2009.
\newblock A logical account of formal argumentation.
\newblock {\em Studia Logica}, {\it 93}, 2-3, 109.

\bibitem[Dung, 1995a]{dung95argumentation}
{\sc Dung, P.} 1995a.
\newblock An argumentation procedure for disjunctive logic programs.
\newblock {\em Journal of Logic Programming}, {\it 24}a, 151--177.

\bibitem[Dung, 1995b]{dung1995acceptability}
{\sc Dung, P.~M.} 1995b.
\newblock On the acceptability of arguments and its fundamental role in nonmonotonic reasoning, logic programming and n-person games.
\newblock {\em Artificial intelligence}, {\it 77}b, 2, 321--357.

\bibitem[Dung et~al., 2009]{dung2009assumption}
{\sc Dung, P.~M.}, {\sc Kowalski, R.~A.}, {\sc and} {\sc Toni, F.} 2009.
\newblock Assumption-based argumentation.
\newblock In {\em Argumentation in artificial intelligence}, pp. 199--218. Springer.

\bibitem[Dunne et~al., 2015]{dunne2015characteristics}
{\sc Dunne, P.~E.}, {\sc Dvo{\v{r}}{\'a}k, W.}, {\sc Linsbichler, T.}, {\sc and} {\sc Woltran, S.} 2015.
\newblock Characteristics of multiple viewpoints in abstract argumentation.
\newblock {\em Artificial Intelligence}, {\it 228}, 153--178.

\bibitem[Dvo{\v{r}}{\'a}k et~al., 2019]{dvovrak2019expressive}
{\sc Dvo{\v{r}}{\'a}k, W.}, {\sc Fandinno, J.}, {\sc and} {\sc Woltran, S.} 2019.
\newblock On the expressive power of collective attacks.
\newblock {\em Argument \& Computation}, {\it 10}, 2, 191--230.

\bibitem[Dvo{\v{r}}{\'a}k et~al., 2013]{dvorak2013making}
{\sc Dvo{\v{r}}{\'a}k, W.}, {\sc Gaggl, S.~A.}, {\sc Wallner, J.~P.}, {\sc and} {\sc Woltran, S.} 2013.
\newblock Making use of advances in answer-set programming for abstract argumentation systems.
\newblock In {\em Applications of Declarative Programming and Knowledge Management}, pp. 114--133. Springer.

\bibitem[Dvo{\v{r}}{\'a}k et~al., 2023]{dvovrak2023claim}
{\sc Dvo{\v{r}}{\'a}k, W.}, {\sc Rapberger, A.}, {\sc and} {\sc Woltran, S.} 2023.
\newblock A claim-centric perspective on abstract argumentation semantics: Claim-defeat, principles, and expressiveness.
\newblock {\em Artificial Intelligence}, {\it 324}, 104011.


\bibitem[Eiter et~al., 1997]{eiter97partial}
{\sc Eiter, T.}, {\sc Leone, N.}, {\sc and} {\sc Sacc\'{a}, D.} 1997.
\newblock On the partial semantics for disjunctive deductive databases.
\newblock {\em Ann. Math. Artif. Intell.}, {\it 19}, 1-2, 59--96.

\bibitem[Flouris and Bikakis, 2019]{flouris2019comprehensive}
{\sc Flouris, G.} {\sc and} {\sc Bikakis, A.} 2019.
\newblock A comprehensive study of argumentation frameworks with sets of attacking arguments.
\newblock {\em International Journal of Approximate Reasoning}, {\it 109}, 55--86.

\bibitem[Gorogiannis and Hunter, 2011]{gorogiannis2011instantiating}
{\sc Gorogiannis, N.} {\sc and} {\sc Hunter, A.} 2011.
\newblock Instantiating abstract argumentation with classical logic arguments: Postulates and properties.
\newblock {\em Artificial Intelligence}, {\it 175}, 9-10, 1479--1497.

\bibitem[K{\"o}nig et~al., 2022]{konig22just}
{\sc K{\"o}nig, M.}, {\sc Rapberger, A.}, {\sc and} {\sc Ulbricht, M.} 2022.
\newblock Just a matter of perspective: Intertranslating expressive argumentation formalisms.
\newblock In {\em Proc. of the 9th International Conference on Computational Models of Argument (COMMA)}, pp. 212--223.

\bibitem[Nielsen and Parsons, 2006]{nielsen2006generalization}
{\sc Nielsen, S.~H.} {\sc and} {\sc Parsons, S.} 2006.
\newblock A generalization of {D}ung's abstract framework for argumentation: Arguing with sets of attacking arguments.
\newblock In {\em International Workshop on Argumentation in Multi-Agent Systems}, pp. 54--73. Springer.

\bibitem[Nieves et~al., 2008]{nieves2008preferred}
{\sc Nieves, J.~C.}, {\sc Cort{\'e}s, U.}, {\sc and} {\sc Osorio, M.} 2008.
\newblock Preferred extensions as stable models.
\newblock {\em Theory and Practice of Logic Programming}, {\it 8}, 4, 527--543.

\bibitem[Przymusinski, 1990]{przymusinski90well-founded}
{\sc Przymusinski, T.} 1990.
\newblock The well-founded semantics coincides with the three-valued stable semantics.
\newblock {\em Fundamenta Informaticae}, {\it 13}, 4, 445--463.

\bibitem[Rapberger, 2020]{rapberger2020defining}
{\sc Rapberger, A.}
\newblock Defining argumentation semantics under a claim-centric view.
\newblock In {\em STAIRS@ ECAI} 2020.

\bibitem[Rocha and Cozman, 2022a]{rocha2022bipolar}
{\sc Rocha, V. H.~N.} {\sc and} {\sc Cozman, F.~G.}
\newblock Bipolar argumentation frameworks with explicit conclusions: Connecting argumentation and logic programming.
\newblock In {\em CEUR Workshop Proceedings} 2022a, volume 3197, pp. 49--60. CEUR-WS.

\bibitem[Rocha and Cozman, 2022b]{rocha2022credal}
{\sc Rocha, V. H.~N.} {\sc and} {\sc Cozman, F.~G.}
\newblock A credal least undefined stable semantics for probabilistic logic programs and probabilistic argumentation.
\newblock In {\em Proceedings of the International Conference on Principles of Knowledge Representation and Reasoning} 2022b, volume~19, pp. 309--319.

\bibitem[S{\'a} and Alc{\^a}ntara, 2019]{sa2019interpretations}
{\sc S{\'a}, S.} {\sc and} {\sc Alc{\^a}ntara, J.} 2019.
\newblock Interpretations and models for assumption-based argumentation.
\newblock In {\em Proceedings of the 34th ACM/SIGAPP Symposium on Applied Computing}, pp. 1139--1146.

\bibitem[S{\'a} and Alc{\^a}ntara, 2021a]{sa2021abstract}
{\sc S{\'a}, S.} {\sc and} {\sc Alc{\^a}ntara, J.} 2021a.
\newblock An abstract argumentation and logic programming comparison based on 5-valued labellings.
\newblock In {\em Symbolic and Quantitative Approaches to Reasoning with Uncertainty: 16th European Conference, ECSQARU 2021}, pp. 159--172. Springer.

\bibitem[S{\'a} and Alc{\^a}ntara, 2021b]{sa2021assumption}
{\sc S{\'a}, S.} {\sc and} {\sc Alc{\^a}ntara, J.} 2021b.
\newblock Assumption-based argumentation is logic programming with projection.
\newblock In {\em Symbolic and Quantitative Approaches to Reasoning with Uncertainty: 16th European Conference, ECSQARU 2021}, pp. 173--186. Springer.

\bibitem[Schulz and Toni, 2015]{schulz2015logic}
{\sc Schulz, C.} {\sc and} {\sc Toni, F.} 2015.
\newblock Logic programming in assumption-based argumentation revisited-semantics and graphical representation.
\newblock In {\em 29th AAAI Conf. on Art. Intelligence}.

\bibitem[Toni, 2014]{toni2014tutorial}
{\sc Toni, F.} 2014.
\newblock A tutorial on assumption-based argumentation.
\newblock {\em Argument \& Computation}, {\it 5}, 1, 89--117.

\bibitem[Toni and Sergot, 2011]{toni2011argumentation}
{\sc Toni, F.} {\sc and} {\sc Sergot, M.} 2011.
\newblock Argumentation and answer set programming.
\newblock In {\em Logic programming, knowledge representation, and nonmonotonic reasoning}, pp. 164--180. Springer.

\bibitem[Verheij, 1996]{verheij1996two}
{\sc Verheij, B.} 1996.
\newblock Two approaches to dialectical argumentation: admissible sets and argumentation stages.
\newblock {\em Proc. NAIC}, {\it 96}, 357--368.

\bibitem[Wu et~al., 2009]{wu2009complete}
{\sc Wu, Y.}, {\sc Caminada, M.}, {\sc and} {\sc Gabbay, D.~M.} 2009.
\newblock Complete extensions in argumentation coincide with 3-valued stable models in logic programming.
\newblock {\em Studia logica}, {\it 93}, 2-3, 383.

\end{thebibliography}
\end{document}